%% file: paper.tex
\RequirePackage{fix-cm}
\documentclass[]{bytedance_seed}

\usepackage[toc,page,header]{appendix}
\usepackage{amsmath,amssymb}
\usepackage{graphicx}
\usepackage{amsthm}
\usepackage{xcolor}
\usepackage{colortbl}
\usepackage{xspace}
\usepackage{booktabs}
\usepackage{multirow}
\usepackage{arydshln}
\usepackage{algorithm}
\usepackage{algpseudocode}
\usepackage{enumitem}
\usepackage{fontawesome5}

\AtBeginDocument{\setlength{\parfillskip}{0pt plus .45\textwidth}}
\algrenewcommand\algorithmicrequire{\textbf{Require}}
\algrenewcommand\algorithmicensure{\textbf{Output}}

\newtheorem{opsdprop}{Proposition}[section]
\newtheorem{opsdlemma}[opsdprop]{Lemma}
\newtheorem{opsdassumption}[opsdprop]{Assumption}

\definecolor{vefblue}{RGB}{30,90,170}
\definecolor{vefred}{RGB}{200,60,60}
\definecolor{vefgreen}{RGB}{40,130,90}
\definecolor{vefgrey}{RGB}{120,120,120}
\definecolor{vefyellow}{RGB}{220,180,40}
\definecolor{vefpurple}{RGB}{120,80,170}
\definecolor{vefblue1}{RGB}{49,89,180}
\definecolor{vefblue2}{RGB}{80,135,242}
\definecolor{vefblue3}{RGB}{86,188,199}
\definecolor{vefblue4}{RGB}{148,228,221}
\definecolor{vefaccent}{RGB}{75,85,170}

\hypersetup{
  colorlinks=true,
  linkcolor=vefblue3,
  citecolor=vefblue3,
  urlcolor=vefblue3,
}

\colorlet{vefoursrow}{vefpurple!8}
\colorlet{vefbaserow}{vefyellow!12}

\newcommand{\namec}{\textsl{\textbf{\textcolor{vefblue2}{Diffusion}\textcolor{vefblue3}{OPSD}}}}
\newcommand{\namew}{DiffusionOPSD}
\makeatletter
\newcount\opsd@uid
\newcount\opsd@shippage
\newif\ifopsd@firstread
\opsd@firstreadtrue
\newif\ifopsd@intabular
\AtBeginEnvironment{tabular}{\opsd@intabulartrue}
\AddToHook{shipout/before}{\global\advance\opsd@shippage\@ne}
\def\opsd@pages{}
\gdef\opsd@old@sig{INIT}
\newcommand{\opsd@reset}{%
  \@for\opsd@p:=\opsd@pages\do{%
    \ifx\opsd@p\@empty\else
      \expandafter\let\csname opsd@best@uid@\opsd@p\endcsname\relax
    \fi
  }%
  \def\opsd@pages{}%
}
\newcommand{\opsd@record}[4]{%
  \@ifundefined{opsd@best@uid@#2}{%
    \edef\opsd@pages{\opsd@pages,#2}%
    \@namedef{opsd@best@uid@#2}{#1}%
    \@namedef{opsd@best@y@#2}{#3}%
    \@namedef{opsd@best@x@#2}{#4}%
  }{%
    \ifnum#3>\numexpr\@nameuse{opsd@best@y@#2}+131072\relax
      \@namedef{opsd@best@uid@#2}{#1}%
      \@namedef{opsd@best@y@#2}{#3}%
      \@namedef{opsd@best@x@#2}{#4}%
    \else
      \ifnum#3<\numexpr\@nameuse{opsd@best@y@#2}-131072\relax
      \else
        \ifnum#4<\@nameuse{opsd@best@x@#2}\relax
          \@namedef{opsd@best@uid@#2}{#1}%
          \@namedef{opsd@best@y@#2}{#3}%
          \@namedef{opsd@best@x@#2}{#4}%
        \fi
      \fi
    \fi
  }%
}
\newcommand{\opsd@checksig}{%
  \def\opsd@cur@sig{}%
  \@for\opsd@p:=\opsd@pages\do{%
    \ifx\opsd@p\@empty\else
      \edef\opsd@w{\@nameuse{opsd@best@uid@\opsd@p}}%
      \edef\opsd@cur@sig{\opsd@cur@sig,\opsd@p:\opsd@w}%
      \ifopsd@firstread
        \global\expandafter\let\csname opsd@win@\opsd@w\endcsname\@empty
      \fi
    \fi
  }%
  \ifopsd@firstread
    \global\let\opsd@old@sig\opsd@cur@sig
  \else
    \ifx\opsd@old@sig\opsd@cur@sig\else
      \global\@tempswatrue
    \fi
  \fi
}
\AtBeginDocument{%
  \global\opsd@firstreadfalse
  \immediate\write\@auxout{\string\opsd@reset}%
}
\AtEndDocument{%
  \clearpage
  \immediate\write\@auxout{\string\opsd@checksig}%
}
\DeclareRobustCommand{\opsd@name}{%
  \ifopsd@intabular
    \namew
  \else
    \leavevmode
    \global\advance\opsd@uid\@ne
    \pdfsavepos
    \protected@write\@auxout{}{%
      \string\opsd@record{\the\opsd@uid}{\noexpand\number\opsd@shippage}{\noexpand\number\pdflastypos}{\noexpand\number\pdflastxpos}%
    }%
    \@ifundefined{opsd@win@\the\opsd@uid}{\namew}{\namec}%
  \fi
}
\DeclareRobustCommand{\name}{\texorpdfstring{\opsd@name}{DiffusionOPSD}\xspace}
\DeclareRobustCommand{\namestar}{\texorpdfstring{\opsd@name\textsuperscript{\normalfont *}}{DiffusionOPSD*}\xspace}
\makeatother
\newcommand{\cmark}{\textcolor{vefgreen}{$\checkmark$}}
\newcommand{\xmark}{\textcolor{vefgrey}{$\times$}}

\newcommand{\shortmonthname}{%
  \ifcase\month\or Jan\or Feb\or Mar\or Apr\or May\or Jun\or
  Jul\or Aug\or Sep\or Oct\or Nov\or Dec\fi}

\title{On-Policy Self-Distillation in Diffusion Models}

\author[]{DiffusionOPSD Team}

\abstract{
Reinforcement learning aligns diffusion and flow models with human preferences and task objectives, yet endpoint rewards do not specify how an intermediate denoising prediction should change.
We introduce \name{}, an on-policy self-distillation framework that converts image-level reward guidance into explicit self-teachers for clean-output predictions at sampled on-policy queries.
At each outer iteration, a frozen exponential moving average self-teacher rolls out trajectories to supply low-noise query states and clean-output anchors.
Normalized local reward gradients then synthesize bounded positive and negative self-teachers around each anchor.
The student policy distills these detached self-teachers under a finite update budget before an online update refreshes the self-teacher, admitting an exact single-branch variant \name{}$^\ast$.
Separating self-teacher synthesis from student backpropagation removes the decoder and reward model from student backward passes and makes synthesis and realized gains separately measurable.
On SD$3.5$-M and step-distilled Z-Image-Turbo, \name{} and \name{}$^\ast$ achieve the best held-out score across all $20$ target-reward settings, winning $19$ of $20$ individually with gains up to $44.0\%$ over the strongest baseline.
They also deliver up to $2.57\times$ the base-anchored gain of ReFL and up to $90\%$ single-reward human win rates, while joint multi-reward training transfers zero-shot with up to $81\%$ human win rates.
Separately, comprehensive profiling experiments show that \name{} and \name{}$^\ast$ use up to $63\%$ and $74\%$ fewer GPU-hours per $100$ updates than DiffusionNFT and reach baseline terminal quality up to $13.1\times$ faster.
Overall, \name{} provides a teacher-free paradigm converting rewards into stable intermediate supervision for diffusion post-training.
}

\checkdata[\faClock\ Date]{Aug 26, 2026\hspace{2.5em}\checkdataformat[\faGlobe\ Project Page]{\href{https://DiffusionOPSD.github.io}{\textbf{DiffusionOPSD}}}\hspace{2.4em}\checkdataformat[\faEnvelope\ Correspondence]{Wei Liu, \href{mailto:liuwei.jikun@bytedance.com}{E-mail}}} % \shortdate

\begin{document}
\maketitle

\begin{center}
    \centering
    % \vspace{-0.3cm}
    \includegraphics[width=0.95\linewidth]{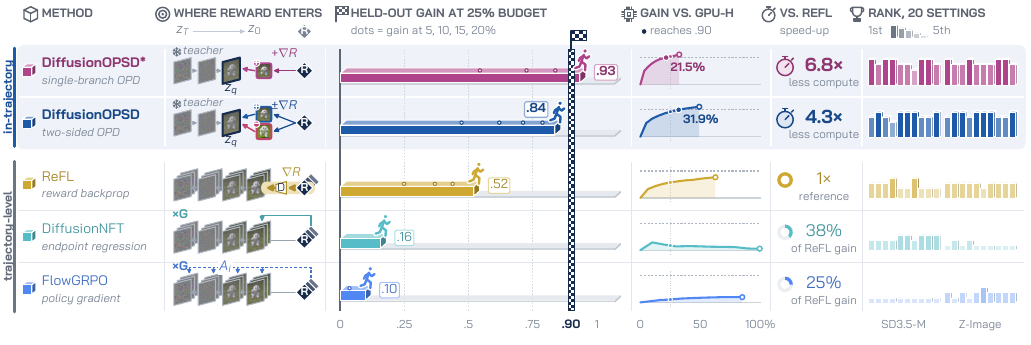}
    \vspace{-0.5cm}
    \captionsetup{hypcap=false}
    \captionof{figure}{\textbf{Paradigm comparison, compute efficiency, and benchmark ranking.} Prior methods apply noisy trajectory credits or end-to-end reward gradients only at decoded final images. In contrast, our methods construct detached reward-gradient self-teachers at intermediate low-noise states, delivering faster convergence and top alignment rankings across comprehensive experiments on two different representative backbones.}
    \label{fig:aggregate_train_test_dynamics}
    \vspace{-0.3cm}
\end{center}

\clearpage
\begin{figure}[t]
    \centering
    \vspace{-0.5cm}
\includegraphics[width=\linewidth,height=0.96\textheight,keepaspectratio]{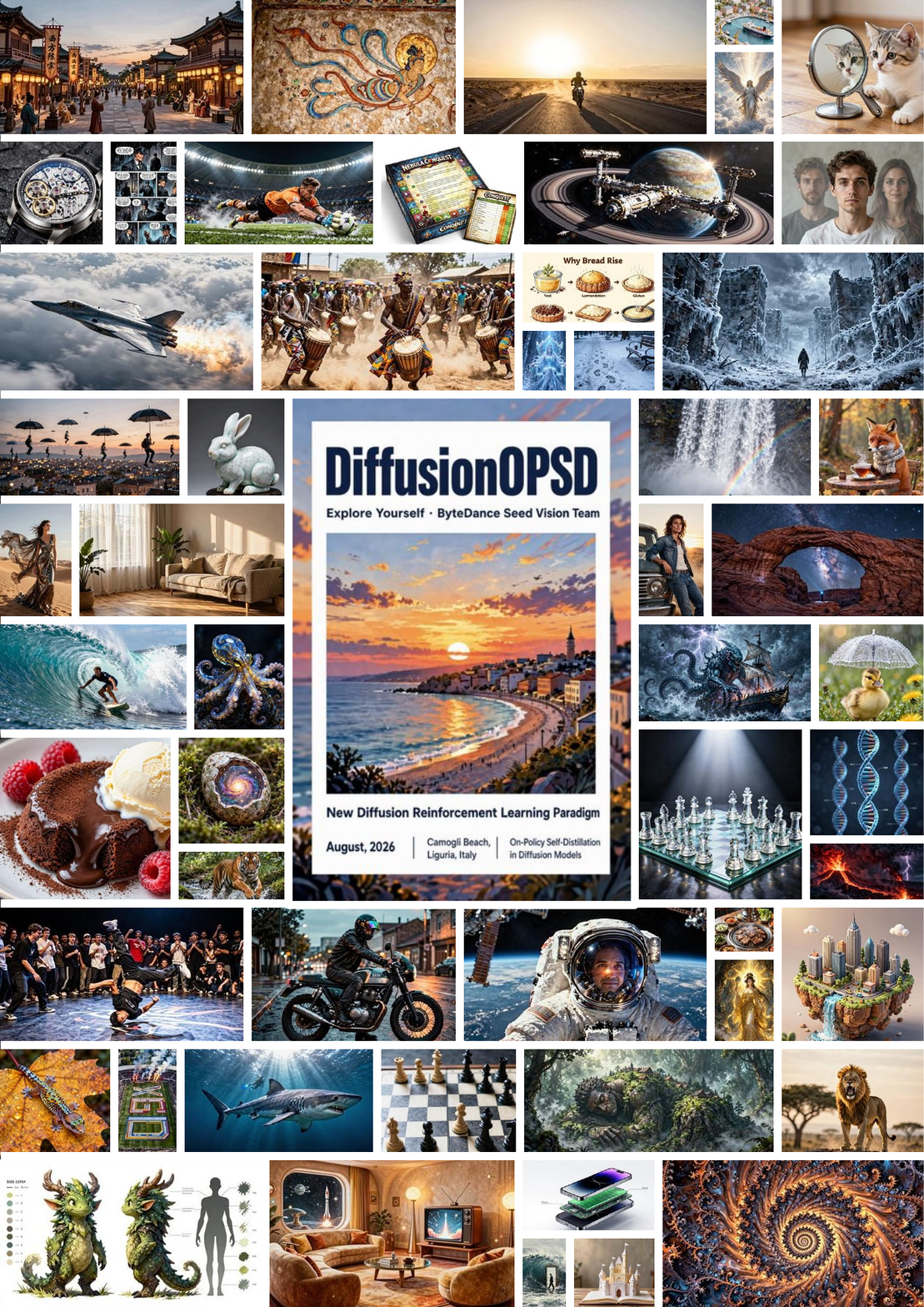}
    \vspace{-0.1cm}
    \caption{\textbf{Qualitative generations.} Representative samples generated by \name{} on held-out text prompts.}
    \label{fig:main}
\end{figure}

\clearpage
\tableofcontents
\newpage

\input{sections/introduction}
\input{sections/relatedwork}
\input{sections/approach}
\input{sections/experiments}

\section{Conclusion}

We introduced \name{}, an on-policy self-distillation framework that converts image-level reward guidance into explicit supervision for intermediate denoising predictions.
At each outer iteration, a frozen EMA self-teacher supplies low-noise query states and clean-output anchors, around which normalized local reward gradients synthesize bounded positive and negative self-teachers.
The student policy distills these detached self-teachers before an online update refreshes the self-teacher, admitting an exact single-branch variant \name{}$^\ast$.
This design avoids treating an endpoint reward as direct supervision for each intermediate step and makes self-teacher synthesis and student distillation separately measurable.
Across SD$3.5$-M and step-distilled Z-Image-Turbo, \name{} and \name{}$^\ast$ achieve the best held-out score across all $20$ target-reward settings, improving over the strongest baseline by up to $44.0\%$.
They also deliver up to $2.57\times$ the base-anchored gain of ReFL and up to $90\%$ single-reward human win rates, while joint multi-reward training transfers zero-shot with up to $81\%$ human win rates.
Matched-budget efficiency comparisons further show that \name{} and \name{}$^\ast$ reduce training GPU-hours by up to $63\%$ and $74\%$ relative to DiffusionNFT and converge up to $13.1\times$ faster.
Overall, our experiments support on-policy self-distillation as an efficient approach to diffusion RL that converts rewards into explicit intermediate supervision, refreshes it online, and keeps synthesis and distillation separately measurable.

\clearpage

\begingroup
\setlength{\parfillskip}{0pt plus 1fil}
\setlength{\bibsep}{2.0pt}
\bibliographystyle{ieeenat_fullname}
\bibliography{main}
\endgroup

\clearpage

\beginappendix

\input{sections/appendix}

\end{document}

%% file: sections/introduction.tex
\section{Introduction}

Diffusion and flow models are widely used for high-fidelity visual generation~\citep{ho2020denoising,song2020score}.
Reinforcement learning is increasingly used to adapt these models to human preferences, semantic fidelity, aesthetics, and finite-horizon task objectives~\citep{black2024training,fan2023dpok,mcallister2026finite,zhu2026diffusion}.
However, endpoint rewards create a structural mismatch between where supervision is observed and where the generative policy acts.
Specifically, each sample is produced through a sequence of interdependent denoising predictions, whereas reward is observed only after the rollout endpoint is decoded~\citep{chow2026weave,zhang2026think,shi2024preference,zhang2026aligning}.
Therefore, improving the final outcome requires translating outcome-level reward into actionable supervision for intermediate denoising queries.

This translation involves two coupled stages, namely \emph{self-teacher synthesis} and \emph{on-policy student distillation}.
First, reward guidance must synthesize a useful local self-teacher prediction at a sampled on-policy query.
Second, finite student updates must realize that detached supervision across shared parameters.
This student realization is nontrivial because the same network is queried repeatedly across noise levels, meaning updates at one query can alter predictions at many others.
Even when a synthesized self-teacher improves reward locally, a finite distillation step may under-realize, rotate, or overshoot the intended change across queries.
Separating self-teacher synthesis from student distillation reveals whether an update fails from a poor local target or from coupled optimization.

Existing paradigms expose reward information at different stages while leaving this intermediate update implicit.
Specifically, FlowGRPO estimates trajectory credit from finite prompt groups and transition likelihood ratios~\citep{liu2026flow}, remaining sensitive to sample budgets, discretization, and stochastic rollouts~\citep{choi2026rethinking}.
Differentiable-reward methods such as ReFL and DRaFT backpropagate image-level rewards directly through clean-output predictions~\citep{xu2023imagereward,clark2024directly}.
However, this coupled backpropagation keeps the diffusion backbone, decoder, and reward model in one memory-heavy graph without a stationary reference, amplifying gradient dispersion across rewards.
In contrast, DiffusionNFT reweights rollout endpoints under a supervised objective~\citep{zheng2026diffusionnft}, yet terminal samples neither specify intermediate updates nor match coarse step-distilled transitions.
Control-theoretic formulations similarly show that the mapping from reward to target governs optimization stability~\citep{yang2026diffusion}.
Meanwhile, on-policy distillation methods such as DanceOPD and DiffusionOPD supervise predictions at low-noise rollout queries~\citep{li2026diffusionopd,zhou2026danceopd}, yet they require pretrained external teachers unavailable in reward post-training.
Motivated by differentiable reward guidance and on-policy distillation, \name{} uses local reward gradients to synthesize bounded positive and negative self-teachers and distills them as detached supervision.

We implement this idea with \name{}, an on-policy self-distillation loop with three stages.
First, a frozen exponential moving average, denoted as EMA, self-teacher generates on-policy trajectories and supplies low-noise query states and clean-output anchors.
At each sampled query, the velocity prediction maps to an equivalent clean-output prediction that can be decoded and scored by an image-level reward.
Second, normalized reward ascent and descent synthesize bounded positive and negative self-teachers around the self-teacher anchor.
The positive self-teacher specifies a reward-improving prediction, while the negative self-teacher provides a repulsive reference through the negative student branch.
A group-normalized endpoint-reward weight balances their relative distillation strength.
Third, the student policy distills these detached self-teachers under a finite update budget before an EMA update refreshes the self-teacher for the next iteration.
In the single-step symmetric setting, positive and negative displacements become exact opposites and the group weight cancels, simplifying the objective to an exact single-branch loss.
Combining this exact reduction with one trajectory per prompt and query-prefix truncation yields our accelerated variant \name{}$^\ast$.

In Figure~\ref{fig:aggregate_train_test_dynamics}, \name{} and \name{}$^\ast$ achieve the best held-out DrawBench score across all $20$ target-reward settings on SD$3.5$-M and step-distilled Z-Image-Turbo.
Individually, each variant wins $19$ of $20$ settings and improves over the strongest baseline by up to $44.0\%$, with representative generations shown in Figure~\ref{fig:main}.
Moreover, decoupling bounded self-teachers from student backpropagation delivers up to $2.57\times$ the base-anchored gain of ReFL across evaluated rewards.
Separately, matched-budget efficiency comparisons show that \name{} and \name{}$^\ast$ use up to $63\%$ and $74\%$ fewer training GPU-hours than DiffusionNFT and reach baseline terminal quality up to $13.1\times$ faster.
Under joint multi-reward training on PickScore, CLIPScore, and HPSv$2.1$, one SD$3.5$-M student surpasses all multi-reward baselines while retaining specialist-level alignment.
This joint policy also transfers zero-shot to unseen evaluators, compositional GenEval tasks, and blinded human preference with win rates up to $81\%$, while single-reward specialists reach up to $90\%$ human win rates.

Separating self-teacher synthesis from student backpropagation via stop-gradient also allows us to examine both stages independently at shared query states.
Controlled ablations show that local reward-gradient directions and on-policy rollout queries drive synthesis gains, while single-step student updates keep realized shifts conservative and concordant with regenerated rollouts.
Across diverse evaluator families, matched-budget mechanism experiments further isolate how self-distillation improves upon graph detachment alone.
Specifically, gradient normalization, stationary EMA anchoring, and adaptive residual scaling improve held-out alignment by up to $36.5\%$ over coupled ReFL while cutting cross-prompt gradient dispersion by up to $50.5\%$.

We summarize our main contributions below.
\begin{itemize}[leftmargin=*,nosep,topsep=2pt]
    \item \textbf{Method.} We introduce \name{} for teacher-free diffusion reward post-training. It synthesizes bounded positive and negative self-teachers from local reward gradients around an EMA self-teacher, distills them as detached supervision, and admits an exact single-branch variant \name{}$^\ast$.
    \item \textbf{Results.} Across two backbones and ten evaluators, \name{} and \name{}$^\ast$ achieve the best held-out score across all $20$ target-reward settings.
    \item \textbf{Analysis.} Separating self-teacher synthesis from student distillation lets us measure synthesis gain and realized gain independently at shared queries.
\end{itemize}

%% file: sections/relatedwork.tex
\section{Related Work}
\label{sec:related}

\subsection{Reward-to-Supervision Interfaces}
Reward-based diffusion post-training methods differ primarily in how terminal feedback is converted into supervision for sequential denoising predictions.
Preference objectives optimize pairwise comparisons~\citep{rafailov2023dpo,wallace2024diffusion,honavar2025dspo} or relative reward rankings~\citep{yang2024using,liang2025aesthetic}, while instruction-guided editing curates localized semantic pairs~\citep{brooks2023instructpix2pix,yu2025anyedit,bai2024humanedit,chow2026editmgt}.
Policy-gradient methods instead rely on sampled trajectory advantages.
Specifically, FlowGRPO and DanceGRPO assign group-relative credit through reverse likelihood ratios~\citep{liu2026flow,xue2025dancegrpo}, whereas AWM uses advantage-weighted flow matching~\citep{xue2025advantage}.
Endpoint-regression methods such as DiffusionNFT and Reinforce Adjoint Matching~\citep{zheng2026diffusionnft,bergmeister2026ram} fit forward-noised rollout endpoints without reward gradients.
Meanwhile, Reward Score Matching~\citep{lee2026reward}, FewStepRL~\citep{fewsteprl2026}, and Diff-Instruct with Diffused Reward~\citep{wu2026didr} propagate endpoint rewards across noise levels for few-step generators.
Differentiable-reward methods backpropagate image-level rewards directly through clean-output predictions or shortened rollouts~\citep{xu2023imagereward,clark2024directly}.
Despite their different objectives, these interfaces leave the desired intermediate change implicit in advantage weights, rollout endpoints, or coupled backward graphs.
In contrast, \name{} synthesizes bounded positive and negative clean-output self-teachers at on-policy queries, distills them as detached supervision, and rebuilds them as the self-teacher evolves.

\subsection{On-Policy Distillation}
Diffusion distillation trains fast generators through progressive and consistency objectives~\citep{hinton2015distilling,salimans2022progressive,song2023consistency}, latent consistency~\citep{luo2023latent,luo2023lcmlora}, or distribution-matching losses~\citep{yin2024one,sauer2024adversarial,sauer2024ladd,lin2025apt}.
Few-step methods either learn direct multi-step transitions~\citep{boffi2025flowmapmatchingstochastic,kim2024ctm,frans2025shortcut,liu2024instaflow} or match generated sample distributions~\citep{yin2024improved,luo2023diffinstruct,zhou2024sid,luo2024sim}.
To mitigate distribution shift, dataset aggregation and on-policy distillation collect supervision at learner-visited states~\citep{ross2011reduction,agarwal2024gkd,xing2026trust,garrepalli2024ddil}.
Building on this principle, diffusion OPD methods supervise velocity fields~\citep{fang2026flow,li2026diffusionopd,zhou2026danceopd} and continuous rollout trajectories~\citep{li2026rl,flowcts2026,smd2026}, while Self Forcing trains on generated rollout histories in autoregressive video~\citep{huang2025selfforcing}.
However, multi-capability OPD requires pretrained specialist teachers, whereas existing self-distillation variants rely on privileged context~\citep{liu2026opsd,yuan2026vision,dopsd2026} or rollout-level reward comparisons~\citep{selfopd2026}.
To fill this gap, \name{} anchors each target at the EMA self-teacher's own clean-output prediction, updates it with local reward gradients, and refreshes it online without an external teacher.

\subsection{Reward Gradients}
Classifier and classifier-free guidance steer sampling through conditional score composition~\citep{dhariwal2021diffusion,ho2022classifier}.
During post-training, AlignProp, video alignment methods, and latent reward surrogates such as LaSRO~\citep{prabhudesai2023aligning,prabhudesai2024video,jia2025lasro} translate differentiable rewards into parameter updates alongside ReFL and DRaFT.
More recently, LeapAlign~\citep{liang2026leapalign} and FlowBP~\citep{wang2026flowbp} backpropagate through shortened or cached reverse steps, while SQDF~\citep{kang2026sqdf} reparameterizes soft Q-gradients with replay buffers.
In parallel, continuous-time control~\citep{domingo2025adjoint,frans2025diffusion} and score or value guidance~\citep{zhao2025score,liu2026value} connect reward gradients to policy improvement.
While infinitesimal positive-target distillation shares the first-order ascent direction at the anchor, direct and shortened-rollout backpropagation retain the decoder and reward graph during parameter updates.
To address these issues, \name{} uses reward gradients only to synthesize bounded detached self-teachers at self-teacher queries and distills them without retaining the decoder or reward graph.

%% file: sections/approach.tex
\begin{figure}[t]
    \centering
    \vspace{-0.6cm}
    \includegraphics[width=\linewidth]{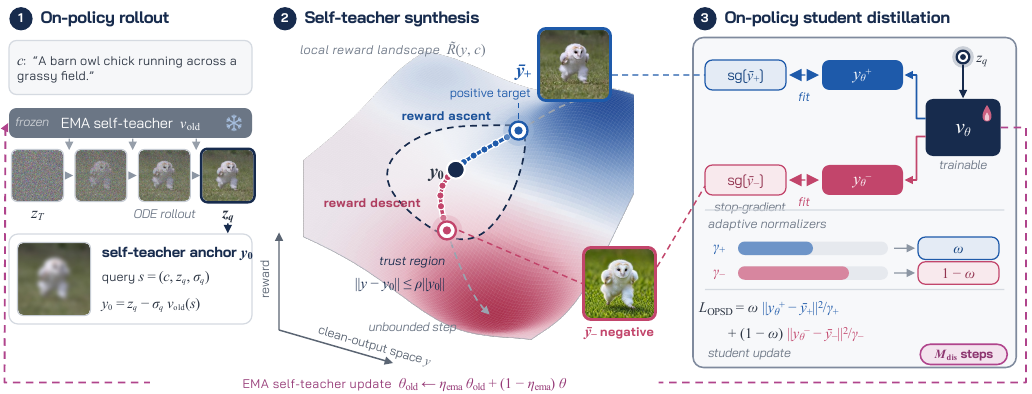}
    \vspace{-0.3cm}
    \caption{\textbf{Overview of the \name{} framework.} First, a frozen exponential moving average self-teacher rolls out on-policy trajectories to supply low-noise queries and clean-output anchors. Second, normalized local reward ascent and descent synthesize bounded positive and negative self-teachers around each anchor. Third, the student distills both detached targets before an online exponential moving average update refreshes the self-teacher.}
    \label{fig:diffusionopsd_overview}
\end{figure}

\section{Approach}
\label{sec:approach}

\name{} treats diffusion reward post-training as teacher-free on-policy self-distillation~\citep{li2026diffusionopd,zhou2026danceopd}.
As illustrated in Figure~\ref{fig:diffusionopsd_overview}, it converts image-level reward gradients into detached clean-output self-teachers around a frozen EMA self-teacher anchor and rebuilds those targets as the self-teacher evolves.
Separating self-teacher synthesis from student backpropagation removes the decoder and reward model from student backward passes and makes synthesis and realized gains separately measurable.
We first define the clean-output coordinate in Section~\ref{sec:clean_coordinate} and on-policy anchor collection in Section~\ref{sec:behavior_anchor}.
After that, we introduce bounded self-teacher synthesis in Section~\ref{sec:target_construction} and detached dual-branch student distillation with our single-branch variant \name{}$^\ast$ in Section~\ref{sec:finite_fitting}.
Finally, Section~\ref{sec:online_self_distillation} completes the online loop via EMA self-teacher updates.

\subsection{Clean-Output Prediction}
\label{sec:clean_coordinate}

Let \(v_\theta\) denote the student rectified-flow velocity network~\citep{liu2022flow,lipman2022flow} and \(v_{\mathrm{old}}\equiv v_{\theta_{\mathrm{old}}}\) denote the frozen EMA self-teacher within the current outer iteration.
For prompt \(\mathbf c\), noisy latent \(z_\sigma\), and noise level \(\sigma\in(0,1]\), we form the query state \(s=(\mathbf c,z_\sigma,\sigma)\).
Under the linear rectified-flow path \(z_\sigma=(1-\sigma)y+\sigma\epsilon\) between clean latent \(y\) and Gaussian noise \(\epsilon\) with target velocity \(v=\epsilon-y\), each query admits the clean-output prediction
\begin{equation}
    \boxed{
    y_\theta(s)
    = z_\sigma-\sigma v_\theta(s),
    \qquad
    v_\theta(s)=\frac{z_\sigma-y_\theta(s)}{\sigma}
    }
    \label{eq:local_action}
\end{equation}
\input{tables/algorithm}
Here \(y_\theta(s)\) is the clean image latent corresponding to the velocity prediction at query \(s\).
Passing \(y_\theta(s)\) through the frozen latent decoder \(D\)~\citep{rombach2022high} defines the decoded local reward \(\widetilde R(y,\mathbf c)=R(D(y),\mathbf c)\).
At a fixed query \(s\), displacements in clean-output space and velocity space satisfy
\begin{equation}
    \delta y=-\sigma\,\delta v,
    \qquad
    \|\delta v\|_2=\frac{\|\delta y\|_2}{\sigma},
    \label{eq:clean_velocity_geometry}
\end{equation}
meaning equal clean-output radii induce larger velocity shifts at lower noise \(\sigma\).

We distinguish three reward quantities throughout the paper.
The endpoint reward scores a completed rollout endpoint and sets the group-normalized distillation weight.
The decoded local reward \(\widetilde R\) scores a clean-output prediction at \(s\) and drives self-teacher synthesis.
The fixed-suffix reward \(F_q\), obtained by continuing remaining denoising steps from \(s\) with \(v_{\mathrm{old}}\), measures synthesis gain before distillation and realized gain after distillation at the same query.

\subsection{On-Policy Rollout and Self-Teacher Anchor}
\label{sec:behavior_anchor}

At outer iteration \(i\), the frozen EMA self-teacher \(v_{\mathrm{old}}\) rolls out \(K\) trajectories per prompt \(\mathbf c\).
Each trajectory \(k\) provides a rollout endpoint \(x_0^k\), endpoint reward \(r_{\mathbf c}^k=R(D(x_0^k),\mathbf c)\), and a detached low-noise query at schedule step \(\sigma_q\) nearest the requested level \(\sigma^\star\), given by
\begin{equation}
    q=\arg\min_j|\sigma_j-\sigma^\star|,
    \qquad
    s=(\mathbf c,z_q,\sigma_q).
    \label{eq:query_state}
\end{equation}
Selecting \(\sigma^\star\) in the low-noise denoising window keeps decoded clean predictions close to completed rollout images.
At query \(s\), the frozen self-teacher defines the stationary clean-output anchor
\begin{equation}
    y_0=z_q
    -\sigma_qv_{\mathrm{old}}(s),
    \label{eq:anchor}
\end{equation}
around which local reward gradients synthesize positive and negative self-teachers.
Collecting \(z_q\) along EMA self-teacher rollouts keeps supervision on-policy as the student evolves.

Endpoint rewards also determine a group-normalized branch weight \(\omega_{\mathbf c}^k\in[0,1]\).
Let \(\bar r_{\mathbf c}=\frac{1}{K}\sum_{j=1}^K r_{\mathbf c}^j\) denote the prompt-group mean reward and \(\widehat\sigma_{\mathcal B_i}\) denote the standard deviation across the rollout batch \(\mathcal B_i\).
Centering \(r_{\mathbf c}^k\) within its prompt group and scaling by \(\widehat\sigma_{\mathcal B_i}\) yields
\begin{equation}
    \omega_{\mathbf c}^k
    =\frac12+\frac12\operatorname{clip}
      \left(\frac{r_{\mathbf c}^k-\bar r_{\mathbf c}}
      {c_{\mathrm{adv}}(\widehat\sigma_{\mathcal B_i}+\epsilon_Z)},-1,1\right).
\label{eq:group_weight}
\end{equation}
Here \(c_{\mathrm{adv}}\) sets the clipping range and \(\epsilon_Z\) prevents division by zero when group rewards are identical.
Per-prompt centering removes prompt-difficulty offsets while \(\widehat\sigma_{\mathcal B_i}\) places all groups on a common scale.
For the \(K=1\) setting with neutral weight \(\omega=1/2\), rollouts can stop directly at \(\sigma_q\) via query-prefix truncation without executing or decoding the remaining suffix.

\subsection{Self-Teacher Synthesis}
\label{sec:target_construction}

To determine the local improvement direction without an external teacher, we differentiate \(\widetilde R\) with respect to the clean-output prediction around \(y_0\).
Reward ascent supplies the positive self-teacher direction, while reward descent supplies the repulsive negative reference.
Let \(\rho>0\) be the relative radius, \(M_{\mathrm{tgt}}\) be the number of synthesis steps, and \(h_{\mathrm{step}}=\rho\|y_0\|_2/M_{\mathrm{tgt}}\) be the step size proportional to anchor norm \(\|y_0\|_2\).
Starting from \(y_+^{(0)}=y_-^{(0)}=y_0\), we apply \(M_{\mathrm{tgt}}\) stabilized normalized gradient steps
\begin{align}
    y_+^{(m+1)}
    &=y_+^{(m)}+
    h_{\mathrm{step}}\frac{\nabla_y\widetilde R(y_+^{(m)},\mathbf c)}
    {\|\nabla_y\widetilde R(y_+^{(m)},\mathbf c)\|_2+\epsilon_g},
    \nonumber\\
    y_-^{(m+1)}
    &=y_-^{(m)}-
    h_{\mathrm{step}}\frac{\nabla_y\widetilde R(y_-^{(m)},\mathbf c)}
    {\|\nabla_y\widetilde R(y_-^{(m)},\mathbf c)\|_2+\epsilon_g},
    \label{eq:target_steps}
\end{align}
for \(m=0,\ldots,M_{\mathrm{tgt}}-1\), where \(\epsilon_g\) stabilizes normalization near stationary points.
When \(\|\nabla_y\widetilde R\|_2\gg\epsilon_g\), unit normalization makes the step length invariant to positive reward rescaling.
Because each increment has norm at most \(h_{\mathrm{step}}\), the triangle inequality keeps every intermediate step within the clean-output ball of radius \(\rho\|y_0\|_2\),
\begin{equation}
    \|y_\pm^{(m)}-y_0\|_2
    \leq m h_{\mathrm{step}}
    \leq\rho\|y_0\|_2.
    \label{eq:target_projection}
\end{equation}
Applying the stop-gradient operator \(\operatorname{sg}(\cdot)\) after \(M_{\mathrm{tgt}}\) steps yields the detached positive and negative self-teachers
\begin{equation}
    \bar y_+=\operatorname{sg}(y_+^{(M_{\mathrm{tgt}})}),
    \qquad
    \bar y_-=\operatorname{sg}(y_-^{(M_{\mathrm{tgt}})}),
    \label{eq:targets}
\end{equation}
which satisfy \(\|\bar y_\pm-y_0\|_2\le\rho\|y_0\|_2\) prior to student distillation.
For one step with anchor gradient \(g_0=\nabla_y\widetilde R(y_0,\mathbf c)\), Taylor expansion yields the local reward change
\begin{equation}
    \widetilde R(y_\pm^{(1)},\mathbf c)-\widetilde R(y_0,\mathbf c)
    =
    \pm h_{\mathrm{step}}\frac{\|g_0\|_2^2}{\|g_0\|_2+\epsilon_g}
    +O(h_{\mathrm{step}}^2).
    \label{eq:first_order_reward_change}
\end{equation}
Appendix~\ref{app:reward_improvement} extends Eq.~\eqref{eq:first_order_reward_change} to multi-step synthesis and proves low-noise fixed-suffix alignment with \(F_q\).

\subsection{On-Policy Student Distillation}
\label{sec:finite_fitting}

During student distillation, we evaluate the student once at query \(s\) to obtain \(y_\theta=z_q-\sigma_qv_\theta(s)\).
With branch coefficient \(\beta>0\), we form positive and negative student branches around \(y_0\), given by
\begin{equation}
    y_\theta^+=\beta y_\theta+(1-\beta)y_0,
    \qquad
    y_\theta^-=(1+\beta)y_0-\beta y_\theta.
    \label{eq:branches}
\end{equation}
These branches let one shared prediction \(y_\theta\) simultaneously approach \(\bar y_+\) and repel \(\bar y_-\).
To normalize residual scales across prompts and evaluators, we define detached adaptive normalizers
\begin{equation}
\begin{aligned}
    \gamma_+&=\max\{
    \operatorname{sg}(\operatorname{mean}|y_\theta^+-\bar y_+|),
    \epsilon_\gamma\},\\
    \gamma_-&=\max\{
    \operatorname{sg}(\operatorname{mean}|y_\theta^- - \bar y_-|),
    \epsilon_\gamma\},
\end{aligned}
\label{eq:adaptive_normalizers}
\end{equation}
where \(\operatorname{mean}|\cdot|\) averages absolute residuals across clean-output coordinates, \(\operatorname{sg}(\cdot)\) blocks denominator gradients, and \(\epsilon_\gamma\) bounds near-zero residuals.
We then minimize the dual-branch distillation objective
\begin{equation}
\boxed{
    \mathcal L_{\mathrm{OPSD}}
    =
    \omega\frac{\operatorname{mean}[(y_\theta^+-\bar y_+)^2]}{\gamma_+}
    +(1-\omega)\frac{\operatorname{mean}[(y_\theta^- - \bar y_-)^2]}{\gamma_-},
}
\label{eq:opsd_loss}
\end{equation}
which matches low-noise velocity distillation up to \(\sigma_q^2\), as shown in Appendix~\ref{app:velocity_mse_equivalence}.
Self-teacher synthesis and student distillation are separated by stop-gradient, enforcing
\begin{equation}
    \nabla_\theta z_q
    =\nabla_\theta\omega
    =\nabla_\theta y_0
    =\nabla_\theta\bar y_\pm=0.
\label{eq:detach_rules}
\end{equation}
The decoder and reward model are differentiated only with respect to the clean-output variable during synthesis, whereas student distillation runs pure diffusion backward passes without decoder, reward, or rollout graphs.

For fixed normalizers, set \(\delta_\theta=y_\theta-y_0\), \(d_+=\bar y_+-y_0\), \(d_-=\bar y_- -y_0\), \(a_+=\omega/\gamma_+\), and \(a_-=(1-\omega)/\gamma_-\).
Completing the square yields the ideal output-space minimizer
\begin{equation}
    \delta_\theta^\star=\frac{a_+d_+-a_-d_-}{\beta(a_++a_-)}.
    \label{eq:branch_closed_form}
\end{equation}
When \(M_{\mathrm{tgt}}\ge 2\), positive and negative paths query distinct points to capture asymmetric reward curvature \(d_+\neq -d_-\) modulated by \(\omega\), as analyzed in Appendix~\ref{app:implicit_branch}.
Conversely, when \(M_{\mathrm{tgt}}=1\) with unit anchor direction \(u_{\mathrm{grad}}\), the two displacements become exact opposites \(d_+=-d_-=d=h_{\mathrm{step}}u_{\mathrm{grad}}\), yielding \(\gamma_+=\gamma_-=\gamma\).
Consequently, \(\omega\) cancels in Eq.~\eqref{eq:opsd_loss}, simplifying the objective to the exact single-branch loss
\begin{equation}
    \mathcal L_{\mathrm{OPSD}^\ast}
    =\frac{\operatorname{mean}\!\left[(\beta\delta_\theta-d)^2\right]}{\gamma},
    \qquad
    \delta_\theta^\star=\frac{d}{\beta}.
    \label{eq:opsd_star_reduction}
\end{equation}
Pairing Eq.~\eqref{eq:opsd_star_reduction} with \(K=1\), \(\omega=1/2\), \(M_{\mathrm{tgt}}=1\), and query-prefix truncation at \(\sigma_q\) defines \name{}$^\ast$.

Differentiating Eq.~\eqref{eq:opsd_star_reduction} over dimension \(d_y\) shows how self-distillation differs from coupled backpropagation,
\[
    \nabla_\theta\mathcal L_{\mathrm{OPSD}^\ast}(\theta)
    =\frac{2\beta}{d_y\gamma}\left(\frac{\partial y_\theta(s)}{\partial\theta}\right)^{\!\top}\!\bigl(\beta(y_\theta(s)-y_0)-d\bigr).
\]
At \(\theta=\theta_{\mathrm{old}}\) with \(\gamma=\operatorname{mean}|d|>\epsilon_\gamma\), the radius \(\rho\|y_0\|_2\) cancels between \(d\) and \(\gamma\), leaving \(-\frac{2\beta}{d_y\operatorname{mean}|u_{\mathrm{grad}}|}J_\theta^\top u_{\mathrm{grad}}\) for Jacobian \(J_\theta=\partial y_\theta(s)/\partial\theta\).
Once \(\theta\neq\theta_{\mathrm{old}}\) or \(M_{\mathrm{tgt}}\ge 2\), online self-distillation departs from normalized reward backpropagation through three mechanisms.
First, evaluating reward gradients at the stationary anchor \(y_0\) introduces the restoring residual \(\beta(y_\theta(s)-y_0)\) that curbs overshoot across shared parameters.
Second, adaptive scaling \(1/\gamma\) moderates cross-prompt gradient variance as student predictions move.
Third, multi-step synthesis captures reward asymmetry while a single distillation step \(M_{\mathrm{dis}}=1\) keeps realized student updates conservative, as analyzed in Appendix~\ref{app:target_update_gap}.

\input{tables/main_results}

\subsection{Online Self-Distillation}
\label{sec:online_self_distillation}

At outer iteration \(i\), the frozen EMA self-teacher \(\theta_{\mathrm{old}}^i\) induces the on-policy query distribution \(Q_{\mathrm{roll},\sigma_q}^{\mathrm{old},i}\) and detached dataset
\begin{equation}
    \mathcal D_i=
    \{(\mathbf c,z_q,\sigma_q,\omega,
    \bar y_+,\bar y_-)\}_{\mathbf c\sim\mathcal C,\;
    z_q\sim Q_{\mathrm{roll},\sigma_q}^{\mathrm{old},i}}.
\label{eq:target_dataset}
\end{equation}
Given \(\mathcal D_i\), we apply \(M_{\mathrm{dis}}\) student optimizer updates, with \(M_{\mathrm{dis}}=1\) by default, written as
\begin{equation}
    \theta_{i+1/2}
    =
    \operatorname{Distill}_{M_{\mathrm{dis}}}
    \bigl(\theta_i;\operatorname{sg}(\mathcal D_i)\bigr).
\label{eq:finite_fit}
\end{equation}
Within outer iteration \(i\), \(\theta_{\mathrm{old}}^i\) remains frozen during trajectory collection, self-teacher synthesis, and all \(M_{\mathrm{dis}}\) updates so queries, anchors, and targets stay stationary.
Afterwards, an EMA update~\citep{polyak1992acceleration} refreshes the self-teacher at cumulative update \(u\) with retention \(\eta_{\mathrm{ema}}^{(u)}\) via
\begin{equation}
    \theta_{\mathrm{old}}^{i+1}
    =
    \eta_{\mathrm{ema}}^{(u)}\theta_{\mathrm{old}}^i
    +(1-\eta_{\mathrm{ema}}^{(u)})\theta_{i+1/2}.
\label{eq:ema_update}
\end{equation}
At iteration \(i+1\), the refreshed self-teacher collects new queries and self-teachers, completing Algorithm~\ref{alg:diffusionopsd}.

%% file: tables/algorithm.tex
\begin{algorithm}[t]
   \caption{Training procedure of \name{}.}
   \label{alg:diffusionopsd}
   \begin{algorithmic}[1]
      \Require Initial policy \(v_\theta\), reward \(R\), decoder \(D\), prompts \(\mathcal C\), group size \(K\), query level \(\sigma^\star\), radius \(\rho\), steps \(M_{\mathrm{tgt}}\), budget \(M_{\mathrm{dis}}\), coefficient \(\beta\), and EMA retentions \(\eta_{\mathrm{ema}},\eta_{\mathrm{ckpt}}\).
      \State Initialize EMA self-teacher \(\theta_{\mathrm{old}}\gets\theta\), \(\theta_{\mathrm{ckpt}}\gets\theta\), and optimizer-update index \(u\gets0\).
      \For{each outer iteration \(i\)}
         \State Select the current prompt batch \(\mathcal C_i\subseteq\mathcal C\).
         \For{each prompt \(\mathbf c\in\mathcal C_i\)}
            \For{\(k=1,\ldots,K\)}
               \State Roll out frozen \(v_{\mathrm{old}}\) to collect endpoint reward \(r_{\mathbf c}^k\) and detached query \((\mathbf c,z_q^k,\sigma_q^k)\) nearest \(\sigma^\star\).
            \EndFor
         \EndFor
         \State Compute \(\omega_{\mathbf c}^k\) from centered prompt rewards and rollout-batch standard deviation via Eq.~\eqref{eq:group_weight}.
         \State Initialize detached self-teacher dataset \(\mathcal D_i\gets\varnothing\).
         \For{each \((\mathbf c,k)\)}
            \State \(y_0\gets z_q^k-\sigma_q^k v_{\mathrm{old}}(\mathbf c,z_q^k,\sigma_q^k)\).
            \State Synthesize bounded self-teachers \(\bar y_+\) and \(\bar y_-\) around \(y_0\) via Eqs.~\eqref{eq:target_steps},~\eqref{eq:target_projection}, and~\eqref{eq:targets}.
            \State Add detached tuple \((\mathbf c,z_q^k,\sigma_q^k,\omega_{\mathbf c}^k,\bar y_+,\bar y_-)\) to \(\mathcal D_i\).
         \EndFor
         \For{\(j=1,\ldots,M_{\mathrm{dis}}\)}
            \State Evaluate student \(y_\theta\) and distillation branches around \(y_0\) on minibatches from \(\mathcal D_i\) via Eq.~\eqref{eq:branches}.
            \State Distill student parameters \(\theta\) using \(\mathcal L_{\mathrm{OPSD}}\) from Eq.~\eqref{eq:opsd_loss}.
            \State \(u\gets u+1\) and \(\theta_{\mathrm{ckpt}}\gets\eta_{\mathrm{ckpt}}^{(u)}\theta_{\mathrm{ckpt}}+(1-\eta_{\mathrm{ckpt}}^{(u)})\theta\).
         \EndFor
         \State \(\theta_{\mathrm{old}}\gets\eta_{\mathrm{ema}}^{(u)}\theta_{\mathrm{old}}+(1-\eta_{\mathrm{ema}}^{(u)})\theta\).
      \EndFor
      \Ensure Checkpoint-averaged reward-aligned student policy \(v_{\theta_{\mathrm{ckpt}}}\).
   \end{algorithmic}
\end{algorithm}

%% file: tables/main_results.tex
\begin{table}[t]
\centering
\vspace{-0.3cm}
\setlength{\tabcolsep}{2.2pt}
\caption{\textbf{Main comparison on held-out DrawBench prompts.} We evaluate single-reward specialists and shared multi-reward policies across public and internal evaluators on both backbones. Bold and underline mark top-two results per matched block; full protocols appear in Appendix~\ref{app:main_result_details}.}
\vspace{-0.3cm}
\label{tab:main_results}
\resizebox{.98\linewidth}{!}{%
\begin{tabular}{lccrrrrrrrrrr}
\toprule
& & & \multicolumn{7}{c}{\textbf{Open-Sourced Evaluators}} & \multicolumn{3}{c}{\textbf{Internal Evaluators}} 
\\
\cmidrule(lr){4-10}\cmidrule(lr){11-13}
\textbf{Model} & \textbf{Reward-Specific} & \textbf{Updates} & \textbf{Pick} & \textbf{CLIP} & \textbf{HPSv$2.1$} & \textbf{Aes} & \textbf{ImgR} & \textbf{HPSv$3$} & \textbf{DeQA} & \textbf{AltCLIP} & \textbf{Point} & \textbf{Pair} 
\\
\midrule\midrule
    SDXL$^\ddagger$ & \xmark & - & 22.42 & 0.287 & 0.280 & 5.60 & 0.76 & 1.93 & 4.17 & 0.366 & 0.089 & 0.168 
    \\
    SD$3.5$-L$^\ddagger$ & \xmark & - & 22.91 & 0.289 & 0.288 & 5.50 & 0.96 & 5.04 & 4.26 & 0.393 & 0.171 & 0.361 
    \\
    FLUX.1-dev$^\ddagger$ & \xmark & - & 22.84 & 0.295 & 0.274 & 5.71 & 0.96 & 6.71 & 4.41 & 0.376 & 0.179 & 0.323 
    \\
    \midrule
    SD$3.5$-M w/o CFG$^\ddagger$ & \xmark & - & 20.51 & 0.237 & 0.204 & 5.13 & -0.58 & -6.47 & 3.30 & 0.276 & 0.026 & 0.069 \\
    + CFG$^\ddagger$ & \xmark & - & 22.34 & 0.285 & 0.279 & 5.36 & 0.85 & 2.69 & 4.04 & 0.380 & 0.136 & 0.290 
    \\
    + FlowGRPO$^\ddagger$ & \xmark & \textcolor{vefgrey}{$>5$k} & 22.51 & 0.293 & 0.274 & 5.32 & 1.06 & 2.65 & 4.01 & 0.393 & 0.181 & 0.388 
    \\
    & \xmark & \textcolor{vefgrey}{2k} & 22.41 & 0.290 & 0.280 & 5.32 & 0.95 & 2.67 & 4.06 & 0.387 & 0.151 & 0.314 
    \\
    & \xmark & \textcolor{vefgrey}{4k} & 23.50 & 0.280 & 0.316 & 5.90 & 1.29 & 7.08 & 4.15 & 0.397 & 0.162 & 0.325 
    \\
    + DiffusionNFT$^\ddagger$ & \xmark & \textcolor{vefgrey}{1.7k} & 23.80 & 0.293 & 0.331 & 6.01 & 1.49 & 7.48 & 4.34 & 0.398 & 0.111 & 0.296 
    \\
    \hdashline
    + DiffusionNFT & \xmark & 300 & 23.62 & 0.294 & 0.340 & \underline{6.02} & 1.45 & 8.36 & 4.32 & 0.397 & 0.150 & 0.294
    \\
    + DRaFT & \xmark & 300 & 23.83 & 0.306 & 0.348 & 5.97 & 1.23 & 6.10 & 3.24 & 0.374 & 0.123 & 0.247
    \\
    + ReFL & \xmark & 300 & 22.61 & 0.264 & 0.302 & 6.01 & 0.81 & 4.67 & 3.57 & 0.343 & 0.090 & 0.176
    \\
    \rowcolor{vefblue1!10}+ \name{} & \xmark & 300 & \textbf{25.51} & \textbf{0.333} & \textbf{0.389} & \textbf{6.03} & \underline{1.51} & \underline{12.75} & \textbf{4.40} & \underline{0.399} & \underline{0.170} & \underline{0.345} 
    \\
    \rowcolor{vefblue1!10}+ \name{}$^\ast$ & \xmark & 300 &\underline{25.42}	& \underline{0.332}	& \underline{0.386}& \underline{6.02} & \textbf{1.55} & \textbf{13.04} & \underline{4.38} & \textbf{0.402}&\textbf{0.173}&\textbf{0.370} \\
    \hdashline
    + DanceOPD$^\dagger$ & \xmark & 300 & 23.06 & 0.272 & 0.322 & 5.98 & 1.23 & 7.24 & 4.34 & 0.378 & 0.127 & 0.219 
    \\
    + DiffusionOPD$^\dagger$ & \xmark & 300 & 23.05 & 0.269 & 0.322 & 6.04 & 1.23 & 7.39 & 4.34 & 0.373 & 0.122 & 0.199 
    \\
    + FlowOPD$^\dagger$ & \xmark & 300 & 22.62 & 0.263 & 0.300 & 5.98 & 0.92 & 5.00 & 4.20 & 0.353 & 0.101 & 0.166 \\
    \midrule
+ FlowGRPO & \cmark & 100 & 21.65 & 0.263 & 0.248 & 5.26 & 0.42 & -0.82 & 3.79 & 0.342 & 0.095 & 0.198 \\
+ ReFL & \cmark & 100 & \underline{23.92} & 0.308 & \underline{0.358} & \underline{12.09} & 1.28 & 6.18 & \underline{4.85} & 0.408 & 0.193 & 0.290 \\
+ DiffusionNFT & \cmark & 100 & 23.43 & 0.298 & 0.336 & 9.11 & 1.46 & 13.51 & 4.76 & 0.412 & 0.199 & 0.323 \\
\rowcolor{vefblue1!10}
+ \name{} & \cmark & 100 & \textbf{24.94} & \textbf{0.340} & \textbf{0.390} & 12.08 & \textbf{1.76} & \textbf{15.14} & \textbf{4.94} & \underline{0.450} & \underline{0.214} & \textbf{0.465} \\
\rowcolor{vefblue1!10}
+ \name{}$^\ast$ & \cmark & 100 &\textbf{24.94}& \underline{0.327}&\textbf{0.390}&\textbf{12.10}&\underline{1.73} & \underline{14.79} &4.81&	\textbf{0.458}&	\textbf{0.216}&	\underline{0.463} \\
\midrule
Z-Image-Turbo & \xmark & - & 22.86 & 0.276 & 0.296 & 5.41 & 0.96 & 11.17 & 4.44 & 0.392 & 0.213 & 0.422 \\
+ FlowGRPO & \cmark & 100 & 22.96 & 0.275 & 0.305 & 5.46 & 1.01 & 11.58 & 4.51 & 0.394 & 0.217 & 0.420 \\
+ ReFL & \cmark & 100 & 24.54 & 0.313 & 0.380 & 9.79 & 1.37 & 11.32 & 4.60 & 0.441 & 0.227 & 0.481 \\
+ DiffusionNFT & \cmark & 100 & 22.28 & 0.280 & 0.277 & 6.07 & 0.58 & 11.01 & 3.37 & 0.363 & 0.166 & 0.357 \\
\rowcolor{vefblue1!10}
+ \name{} & \cmark & 100 & \underline{25.15} & \textbf{0.320} & \textbf{0.390} & \underline{10.74} & \textbf{1.79} & \textbf{15.59} & \textbf{4.78} & \underline{0.451} & \underline{0.243} & \textbf{0.551} \\
\rowcolor{vefblue1!10}
+ \name{}$^\ast$ & \cmark & 100 & \textbf{25.28}&\underline{0.314}&\underline{0.386}& \textbf{10.80}&\underline{1.76}& \underline{15.51} & \underline{4.77} & \textbf{0.459} & \textbf{0.249}& \underline{0.549} \\
\bottomrule
\end{tabular}%
}
\end{table}

%% file: sections/experiments.tex
\begin{figure}[t]
    \centering
    \vspace{-0.5cm}
    \includegraphics[width=\linewidth]{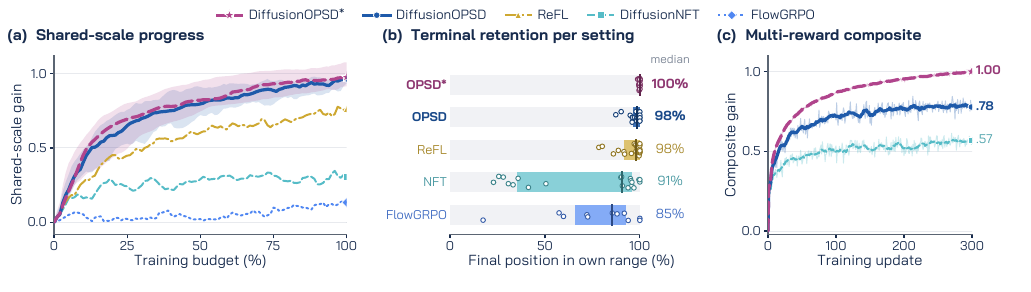}
    \vspace{-0.6cm}
    \caption{\textbf{Normalized training dynamics and terminal retention.} \name{} and \namestar{} achieve the fastest shared-scale progress across single-reward tasks, retain peak gains through the final update without late-stage regression, and maintain smooth multi-reward optimization trajectories.}
    \label{fig:main_table_native_dynamics}
\end{figure}

\section{Experiments}

We evaluate \name{} across alignment quality, training efficiency, and the separation of self-teacher synthesis from student distillation on SD$3.5$-M and step-distilled Z-Image-Turbo.
Sections~\ref{sec:complete_results} and~\ref{sec:efficiency} examine held-out accuracy and matched-budget compute frontiers, whereas Sections~\ref{sec:optimization_regimes} and~\ref{sec:qualitative_results} study optimization stability, multi-reward transfer, and human preference.
Sections~\ref{sec:target_construction_results} and~\ref{sec:ablations} then evaluate self-teacher synthesis, student realization, and component ablations, with full protocols in Appendix~\ref{app:main_result_details}.

\subsection{Main Results}
\label{sec:complete_results}

\noindent\textbf{Evaluation protocol.} We fine-tune LoRA policies on SD$3.5$-M~\citep{stability2024sd35medium,esser2024scaling} and step-distilled Z-Image-Turbo~\citep{cai2025z} using Pick-a-Pic prompts~\citep{kirstain2023pick} and evaluate held-out DrawBench prompts~\citep{saharia2022photorealistic}.
Our ten-evaluator suite comprises seven public evaluators, namely PickScore, CLIPScore, HPSv$2.1$, Aesthetic, ImageReward, HPSv$3$, and DeQA, alongside three internal preference models, namely AltCLIP, VLM-Pointwise, and VLM-Pairwise.
In Table~\ref{tab:main_results}, \cmark\ rows train one dedicated specialist per column compared with FlowGRPO~\citep{liu2026flow}, ReFL~\citep{xu2023imagereward}, and DiffusionNFT~\citep{zheng2026diffusionnft}.
Conversely, \xmark\ rows evaluate a single checkpoint across all ten columns, covering reference models, jointly trained policies on PickScore, CLIPScore, and HPSv$2.1$, and two-stage distillation baselines.
Within the fixed specialist budget, \name{} and \name{}$^\ast$ also surpass the long-horizon FlowGRPO reference on each target evaluator, as detailed in Appendix~\ref{app:main_result_details}.

\begin{figure}[t]
    \centering
    \vspace{-0.5cm}
    \includegraphics[width=\linewidth]{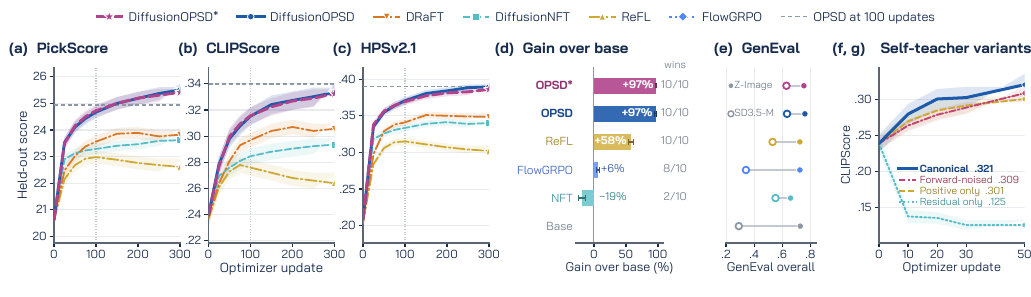}
    \vspace{-0.6cm}
    \caption{\textbf{Multi-reward alignment, step-distilled adaptation, and self-teacher ablations.} (a)--(c) Multi-reward held-out trajectories versus specialist references. (d) and (e) Mean normalized gains on Z-Image-Turbo and compositional accuracy. (f) and (g) Direction and query-source ablations.}
    \label{fig:multireward_curves}
\end{figure}

\noindent\textbf{Reward-specific alignment.} Across all $20$ target-reward settings in Table~\ref{tab:main_results}, \name{} and \name{}$^\ast$ achieve the best held-out score, with each variant winning $19$ of $20$ individually.
Independent multi-seed runs in Figures~\ref{fig:native_dynamics_sd35m}(c) and~\ref{fig:native_dynamics_zimage}(c) confirm tight reproducibility and reduce cross-seed standard deviation by up to $65\%$ compared with competing baselines.
Decoupling bounded self-teacher synthesis from student backpropagation yields strong alignment gains across diverse rewards.
On Z-Image-Turbo, \name{} and \name{}$^\ast$ deliver up to $2.57\times$ the base-anchored gain of ReFL and improve held-out scores over the strongest baseline by $30.7\%$ on ImageReward, $34.6\%$ on HPSv$3$, and $14.6\%$ on VLM-Pairwise.
These alignment improvements hold consistently on the multi-step SD$3.5$-M backbone across all evaluator types.
Relative to the strongest baseline on SD$3.5$-M, \name{} improves CLIPScore by $10.4\%$, ImageReward by $20.5\%$, HPSv$3$ by $12.1\%$, and VLM-Pairwise preference by $44.0\%$.
This VLM-Pairwise gain corresponds to a $14.2$-percentage-point increase in preference probability compared with Seedream $5.0$ Pro references.

\noindent\textbf{Single-step acceleration and multi-step extension.} Combining the exact single-branch loss in Eq.~\eqref{eq:opsd_star_reduction} with one trajectory per prompt and query-prefix truncation makes \name{}$^\ast$ an efficient default under tight budgets.
In Figure~\ref{fig:two_backbone_ten_reward_pareto}, sampling $12\times$ more distinct prompts per update lets \name{}$^\ast$ outperform the strongest Z-Image-Turbo baseline by $10.3\%$ on Aesthetic and $9.7\%$ on VLM-Pointwise.
Meanwhile, two-step dual-branch synthesis in canonical \name{} captures higher-order reward asymmetry and explicit repulsion from degraded predictions.
Consequently, canonical \name{} improves CLIPScore by $6.6\%$ over positive-only distillation in Figure~\ref{fig:multireward_curves}(f) and outperforms \name{}$^\ast$ by up to $4.0\%$ on fine-grained SD$3.5$-M evaluators.

\subsection{Training Efficiency}
\label{sec:efficiency}

\noindent\textbf{Matched-budget training cost.} Removing the latent decoder and reward model from student backpropagation substantially reduces online training cost.
In matched-budget efficiency comparisons in Table~\ref{tab:efficiency_profile}, canonical \name{} reduces training GPU-hours per $100$ updates by $40\%$ on SD$3.5$-M and $63\%$ on Z-Image-Turbo compared with DiffusionNFT.
Relative to FlowGRPO, canonical \name{} similarly reduces training GPU-hours by $34\%$ on SD$3.5$-M and $54\%$ on Z-Image-Turbo.
Combining single-branch distillation with query-prefix truncation in \name{}$^\ast$ delivers further end-to-end compute savings.
Specifically, \name{}$^\ast$ reduces training GPU-hours by $62\%$ on SD$3.5$-M and $74\%$ on Z-Image-Turbo compared with DiffusionNFT while cutting diffusion backward calls by up to $9\times$.

\noindent\textbf{Quality-compute frontiers.} Across all $20$ target-reward benchmarks in Figures~\ref{fig:aggregate_train_test_dynamics} and~\ref{fig:two_backbone_ten_reward_pareto}, \name{} and \name{}$^\ast$ establish the quality-compute Pareto frontier, reaching $0.84$ and $0.93$ normalized held-out gain at $25\%$ budget, crossing $0.90$ gain within $31.9\%$ and $21.5\%$ of the full compute budget, and matching terminal ReFL quality with $4.3\times$ and $6.8\times$ less compute on average.
On individual SD$3.5$-M benchmarks, \name{}$^\ast$ reaches the terminal quality of DiffusionNFT up to $13.1\times$ faster and ReFL up to $9.0\times$ faster, as detailed in Appendix~\ref{app:efficiency_profile}.
Consistent compute efficiency also holds on the step-distilled Z-Image-Turbo generator.
Specifically, \name{}$^\ast$ matches terminal ReFL quality up to $3.2\times$ faster on Z-Image-Turbo and continues improving to higher final plateaus across all ten evaluators.

\begin{figure}[p]
    \centering
    \vspace{-0.4cm}
    \includegraphics[width=\linewidth]{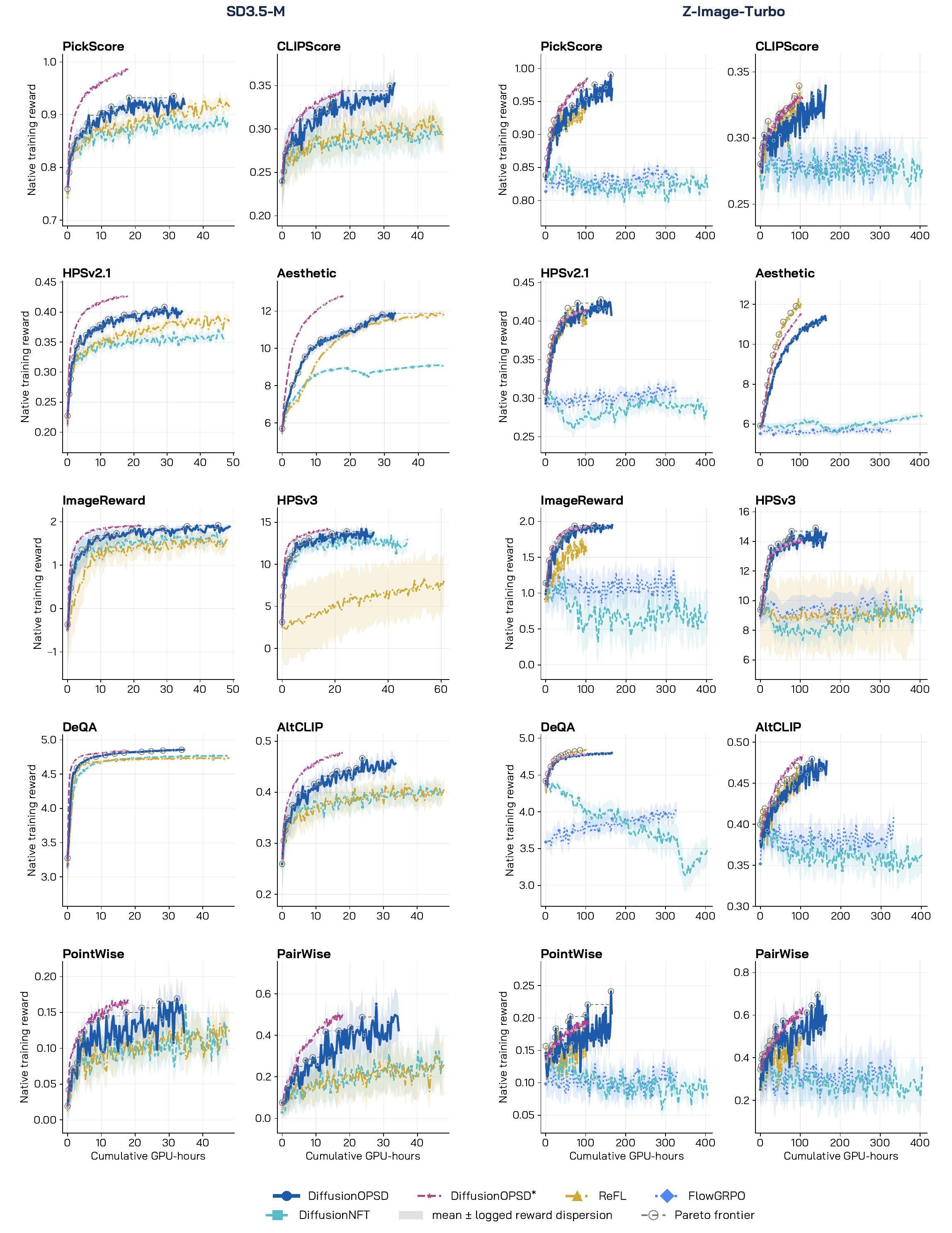}
    \caption{\textbf{Native training reward progression.} \name{} and \namestar{} improve monotonically across all objectives on both backbones, whereas endpoint regression degrades below initialization on most tasks of the step-distilled generator. Shaded bands denote within-prompt reward dispersion.}
    \label{fig:native_gpu_pareto_20reward}
\end{figure}

\begin{figure}[p]
    \centering
    \vspace{-0.4cm}
    \includegraphics[width=\textwidth,height=0.92\textheight,keepaspectratio]{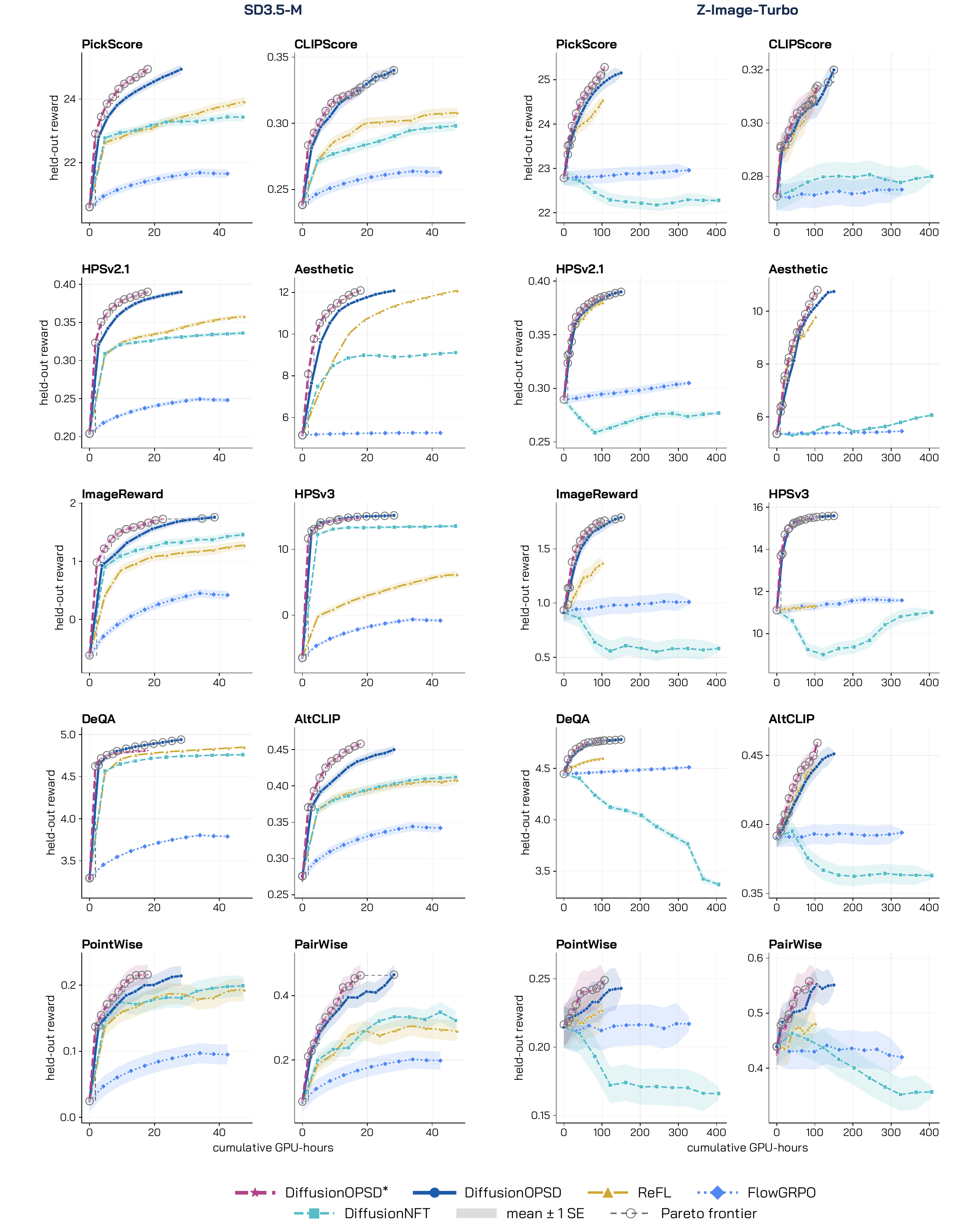}
    \vspace{-0.2cm}
    \caption{\textbf{Held-out alignment quality versus training compute.} \name{} and \namestar{} establish the Pareto frontier across benchmarks on both backbones; markers and shaded bands denote means and standard errors.}
    \label{fig:two_backbone_ten_reward_pareto}
\end{figure}

\subsection{Training Stability and Generalization}
\label{sec:optimization_regimes}

\noindent\textbf{Target retention and transfer.} Anchoring student distillation around a smoothly refreshed EMA self-teacher prevents late-stage reward regression.
In Figure~\ref{fig:main_table_native_dynamics}, \name{} and \name{}$^\ast$ retain medians of $98\%$ and $100\%$ of peak within-run gains with monotonic curves in Figure~\ref{fig:native_gpu_pareto_20reward}, exceeding FlowGRPO by $13$ percentage points and ReFL's lower tail by $20$ percentage points.
Single-reward specialists trained on text-grounded objectives also transfer positively to unseen evaluators rather than overfitting narrow reward quirks.
Paired evaluations in Appendix~\ref{app:main_result_details} show that SD$3.5$-M specialists improve all three non-target evaluators and outperform both ReFL and DiffusionNFT.

\noindent\textbf{Adaptation on a step-distilled generator.} Querying native on-policy states enables reliable post-training on few-step distilled models without access to their original distillation teachers.
On Z-Image-Turbo, \name{} and \name{}$^\ast$ improve all ten objectives with a $97\%$ mean normalized gain in Figure~\ref{fig:multireward_curves}(d), exceeding ReFL by $39$ percentage points and FlowGRPO by $91$ percentage points.
In contrast, regressing directly onto forward-noised rollout endpoints suffers from transition mismatch under compressed few-step schedules.
DiffusionNFT drops below the base model on eight of ten tasks with a $19\%$ mean loss and degrades GenEval by $9.5\%$ in Figure~\ref{fig:multireward_curves}(e), whereas \name{} improves GenEval by $4.5\%$.

\begin{figure}[p]
    \centering
    \vspace{-0.7cm}
    \includegraphics[width=\linewidth]{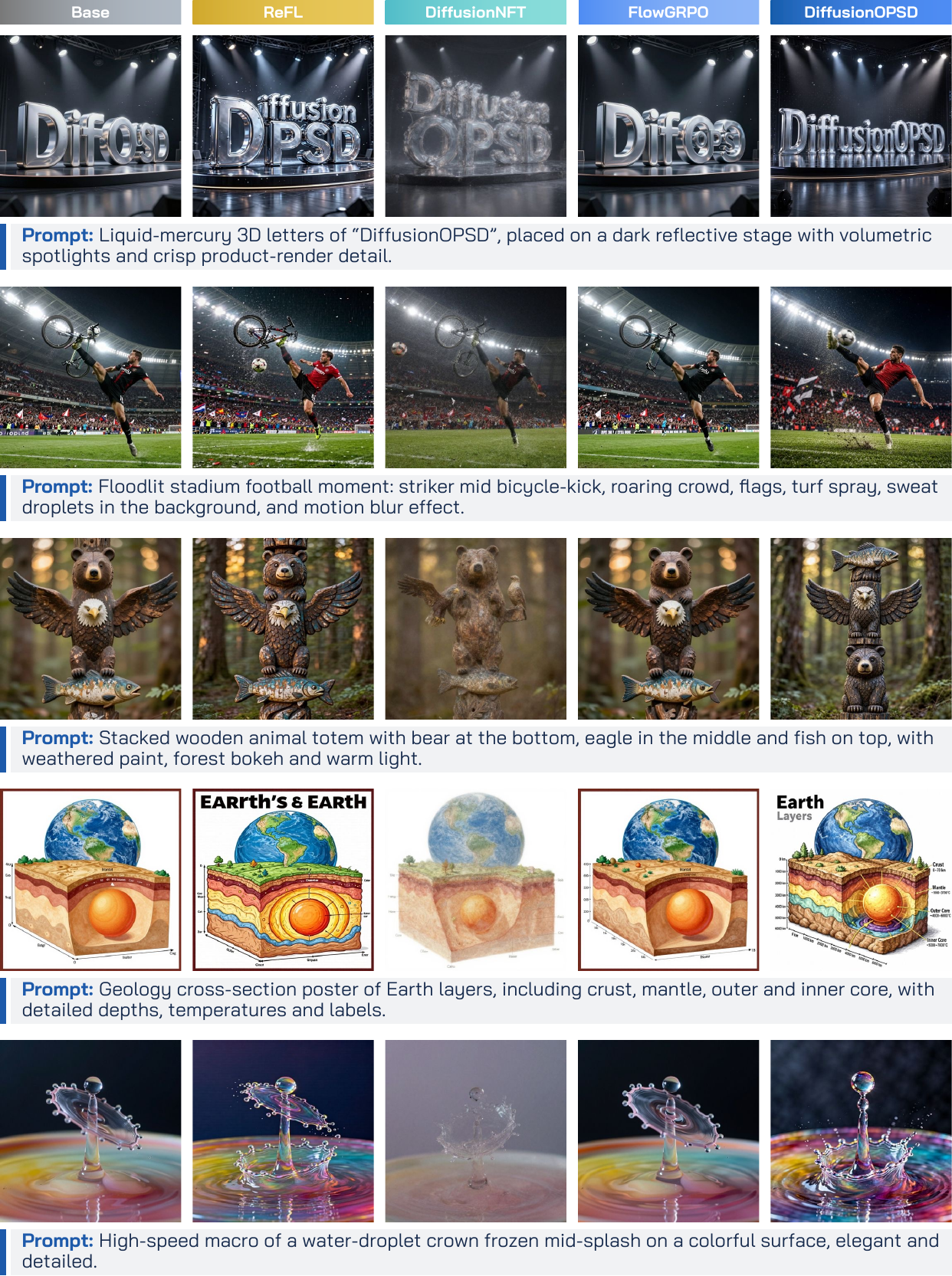}
    \vspace{-0.6cm}
    \caption{\textbf{Held-out qualitative comparisons.} \name{} preserves requested typography, motion, and fine structural detail more reliably than competing post-training baselines.}
    \label{fig:main_qua_01}
\end{figure}

\begin{figure}[p]
    \centering
    \vspace{-0.7cm}
    \includegraphics[width=\linewidth]{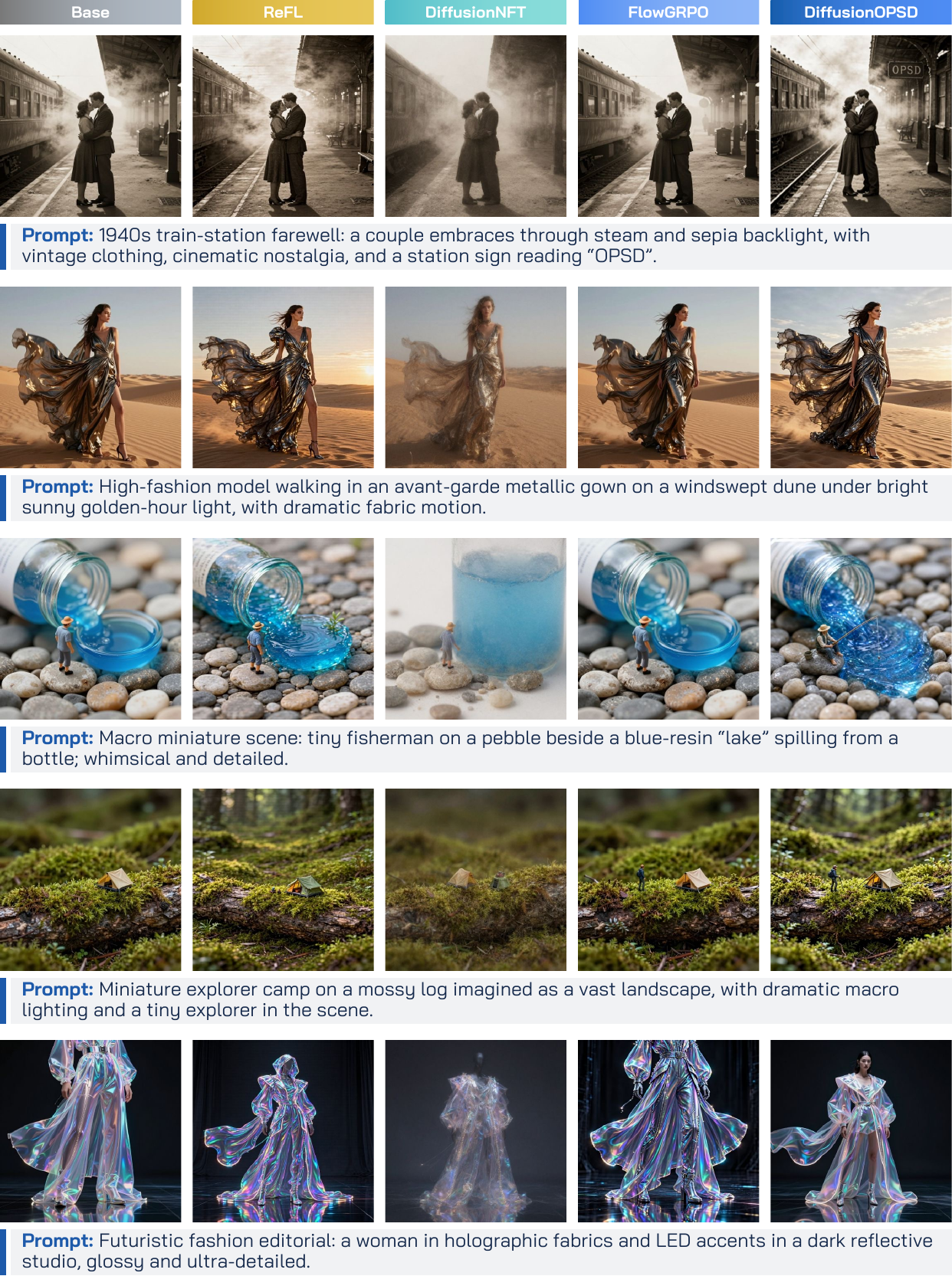}
    \vspace{-0.6cm}
    \caption{\textbf{Motion and spatial composition.} \name{} maintains coherent subject actions, multi-object layouts, and natural surface textures without structural distortion or visual artifacts.}
    \label{fig:main_qua_02}
\end{figure}

\begin{figure}[p]
    \centering
    \vspace{-0.7cm}
    \includegraphics[width=\linewidth]{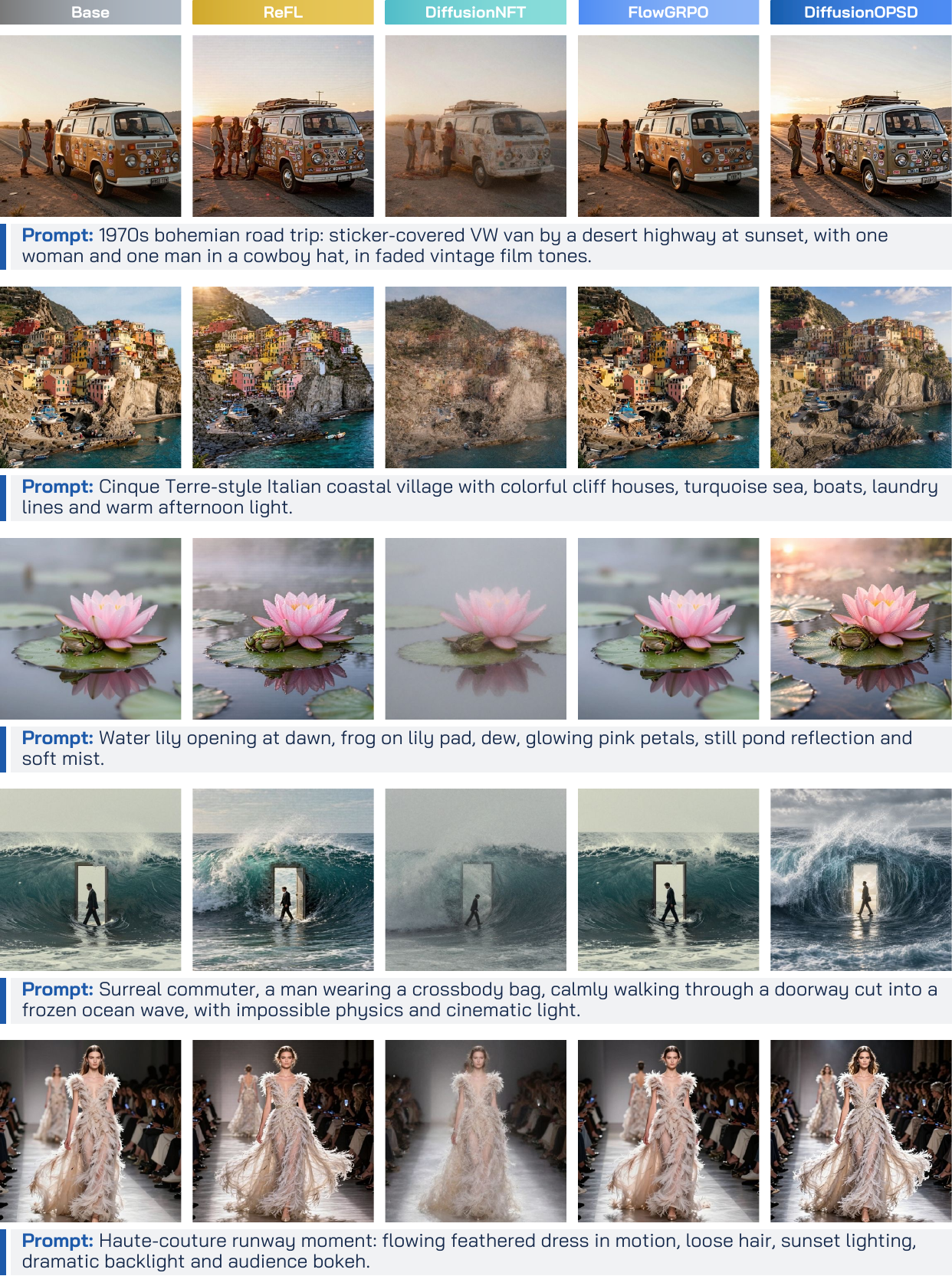}
    \vspace{-0.6cm}
    \caption{\textbf{Complex scenes, material rendering, and scene geometry.} \name{} preserves entity attributes, surface materials, lighting, and fine geometry without high-frequency reward-hacking artifacts.}
    \label{fig:main_qua_03}
\end{figure}

\begin{figure}[t]
    \centering
    \vspace{-0.4cm}
    \includegraphics[width=\linewidth]{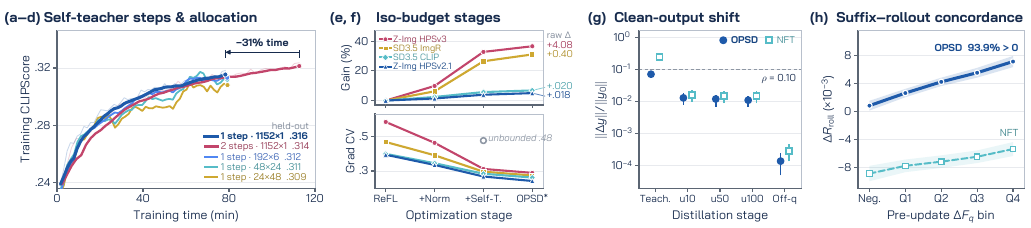}
    \vspace{-0.7cm}
    \caption{\textbf{Self-teacher efficiency, cumulative mechanism attribution, and post-distillation audits.} (a) and (b) Training efficiency and held-out alignment across self-teacher step counts. (c) and (d) Sample efficiency across prompt-group allocations. (e) and (f) Matched progression from end-to-end reward backpropagation to \namestar{} across evaluator families, tracking held-out gains and gradient dispersion. (g) and (h) Clean-output displacement quantiles and concordance between suffix gains and regenerated rollout rewards.}
    \label{fig:ablation_efficiency}
\end{figure}

\noindent\textbf{Joint multi-reward generalization.} Training one shared \name{} policy jointly on multiple rewards matches dedicated specialists while surpassing one-stage and two-stage baselines.
Under joint training on PickScore, CLIPScore, and HPSv$2.1$ in Table~\ref{tab:main_results}, \name{} outperforms DiffusionNFT by up to $14.4\%$, DRaFT by up to $11.8\%$, ReFL by up to $28.8\%$, and two-stage DanceOPD consolidation~\citep{zhou2026danceopd} by up to $22.4\%$.
Joint self-distillation also converges faster and generalizes zero-shot to unseen evaluators and compositional benchmarks.
In Figure~\ref{fig:multireward_curves}(a), (b), and (c), \name{} surpasses all multi-reward baselines within the early third of training and retains up to $112.9\%$ of specialist gains at completion.
Moreover, \name{} and \name{}$^\ast$ rank top among one-stage methods across all seven unseen metrics, improving unseen HPSv$3$ by up to $56.0\%$ and unguided GenEval~\citep{ghosh2023geneval} accuracy by $117.9\%$ in Figure~\ref{fig:multireward_curves}(e).

\subsection{Qualitative Results}
\label{sec:qualitative_results}

\noindent\textbf{Blinded human preference.} Blinded pairwise human evaluations confirm consistent perceptual gains over all competing post-training paradigms.
Across held-out prompts on both backbones in Figure~\ref{fig:ablation_robustness}(b) and (c) and Appendix~\ref{app:ablation_details}, annotators prefer \name{} over each baseline, reaching a $90\%$ win rate over DiffusionNFT and net margins up to $26$ percentage points over ReFL.
Multi-axis ratings on the jointly trained policy further verify genuine visual improvements rather than reward-hacking artifacts.
In Figure~\ref{fig:ablation_robustness}(d), \name{} achieves win rates reaching $78\%$ on prompt fidelity, $77\%$ on image quality, and $81\%$ on artifact-free realism compared with coupled DRaFT and ReFL.

\begin{figure}[t]
    \centering
    \vspace{-0.3cm}
    \includegraphics[width=\linewidth]{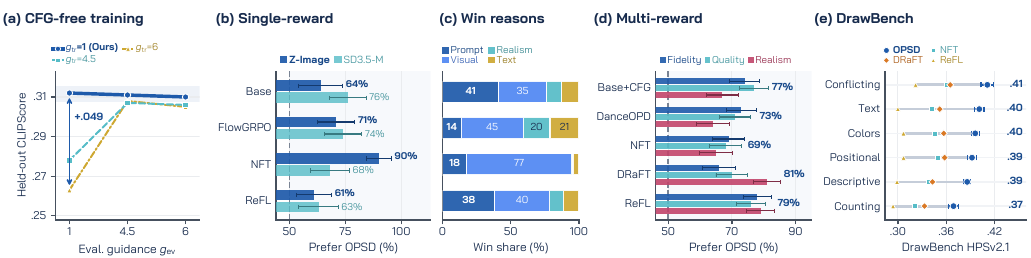}
    \vspace{-0.5cm}
    \caption{\textbf{Guidance controls, human preference, and category alignment.} (a) Held-out alignment across classifier-free guidance scales. (b), (c), and (d) Blinded single-reward and multi-axis human preference win rates with bootstrap confidence intervals. (e) Category-level alignment breakdown across held-out prompt types.}
    \label{fig:ablation_robustness}
\end{figure}

\noindent\textbf{Category breakdown and qualitative comparisons.} Compositional alignment gains hold consistently across diverse semantic prompt categories.
Across all six DrawBench categories in Figure~\ref{fig:ablation_robustness}(e), jointly trained \name{} outperforms DRaFT by up to $13.9\%$ and ReFL by up to $33.7\%$ in HPSv$2.1$ while more than doubling unseen GenEval accuracy in Appendix~\ref{app:ablation_details}.
Side-by-side visual comparisons in Figures~\ref{fig:main_qua_01},~\ref{fig:main_qua_02}, and~\ref{fig:main_qua_03} corroborate these quantitative improvements.
Specifically, \name{} preserves rendered typography, rare concept bindings, coherent motion, and natural surface materials without the attribute omission of DiffusionNFT or high-frequency oversaturation of ReFL.

\subsection{Self-Teacher Synthesis and Student Distillation}
\label{sec:target_construction_results}
\label{sec:target_update_gap}

\noindent\textbf{Self-teacher direction and query source.} Separating self-teacher synthesis from student distillation shows that local reward gradients drive target-level improvement.
In controlled synthesis ablations in Figure~\ref{fig:multireward_curves}(f), two-branch reward-gradient self-teachers improve held-out CLIPScore by $34.3\%$ over initialization and by $6.6\%$ over positive-only targets.
In contrast, rollout-endpoint residuals degrade CLIPScore by $47.7\%$, whereas matched-radius random directions stay within $3.8\%$ of initialization in Figures~\ref{fig:ablation_direction_qual} and~\ref{fig:ablation_progress_direction}.
Sampling queries directly along on-policy self-teacher rollouts also outperforms forward-noised endpoint states.
In Figure~\ref{fig:multireward_curves}(g), on-policy rollout queries improve CLIPScore by $3.9\%$ over forward-noised controls by avoiding forward-reverse state mismatch.

\noindent\textbf{Realized shifts and rollout concordance.} Low-noise local reward gradients align closely with fixed-suffix continuation gradients while keeping realized student updates conservative.
In controlled same-query experiments in Figure~\ref{fig:ablation_efficiency}(g) and Table~\ref{tab:post_fitting_audit} of Appendix~\ref{app:finite_fitting_audit}, positive gradient alignment holds on $93.4\%$ of SD$3.5$-M queries and $96.1\%$ of Z-Image-Turbo queries.
Bounded distillation further keeps $100\%$ of synthesized and realized shifts inside the nominal radius and reduces off-query drift by $51.7\%$ compared with DiffusionNFT.
Same-query fixed-suffix gains also translate monotonically into full-trajectory reward improvements from pure noise.
In Figure~\ref{fig:ablation_efficiency}(h), regenerated rollouts achieve positive reward gains on $93.9\%$ of held-out prompts and rise monotonically across pre-update suffix quartiles, whereas DiffusionNFT induces losses in each quartile.

\noindent\textbf{Same-query coupling and mechanism attribution.} Bounded self-teacher detachment simultaneously reduces policy memory and stabilizes cross-prompt parameter gradients.
At shared query states in Table~\ref{tab:post_fitting_audit}, detaching targets without norm bounds inflates cross-prompt gradient dispersion by $19.3\%$ over coupled ReFL.
Bounded \name{} distillation instead cuts peak policy VRAM by $15.4\%$, reduces gradient dispersion by $34.3\%$, and increases the positive fixed-suffix rate by $52.5$ percentage points over DiffusionNFT.
Matched-budget mechanism ablations across diverse evaluator families further isolate the cumulative role of each self-distillation component.
In Figure~\ref{fig:ablation_efficiency}(e) and (f), per-sample gradient normalization improves ReFL by up to $9.8\%$, while stationary EMA anchoring accounts for up to $85.8\%$ of the remaining gain via its restoring residual.
Adding adaptive residual scaling in \name{}$^\ast$ outperforms coupled ReFL by up to $36.5\%$ and reduces gradient dispersion by up to $50.5\%$.

\subsection{Ablations}
\label{sec:ablations}

\noindent\textbf{Self-teacher steps and prompt allocation.} Reducing self-teacher synthesis to one step decreases training cost with negligible quality change.
This single-step formulation cuts measured training time by $30.5\%$ relative to two-step synthesis in Figure~\ref{fig:ablation_efficiency}(a) and (b) and Appendix~\ref{app:ablation_details} while matching held-out CLIPScore within $0.4\%$.
Allocating the fixed rollout budget to single-trajectory prompts further improves short-horizon sample efficiency.
In Figure~\ref{fig:ablation_efficiency}(c) and (d), single-trajectory collection covers $24\times$ more distinct prompts per update and increases CLIPScore by $1.5\%$ over grouped collection.

\noindent\textbf{Guidance-free training.} Training without classifier-free guidance outperforms guided training while halving per-step network evaluations.
In Figure~\ref{fig:ablation_robustness}(a), the guidance-free policy preserves accuracy across evaluation scales, whereas guided training degrades unguided inference by up to $15.7\%$ and trails by up to $1.6\%$ at matched scales.
This guidance behavior is analyzed in Appendix~\ref{app:cfg_compensation} and illustrated qualitatively in Figures~\ref{fig:ablation_mechanism_qual} and~\ref{fig:ablation_direction_qual}.

%% file: sections/appendix.tex
\section{Author List}
\label{app:author_list}

\noindent\begin{tabular}{@{}l p{0.8\textwidth}@{}}
Wei Zhou & \quad ByteDance Seed; National University of Singapore 
\\
Xiongwei Zhu & \quad ByteDance Seed 
\\
Lingdong Kong & \quad National University of Singapore 
\\
Bo Chen & \quad ByteDance Seed 
\\
Lei Zhang & \quad University of California, San Diego 
\\
Yongyuan Liang & \quad University of Maryland, College Park 
\\
Xiaoxia Hou & \quad ByteDance Seed 
\\
Tian Ye & \quad The Hong Kong University of Science and Technology, Guangzhou 
\\
Xian Sun & \quad Duke University 
\\
Yingshuo Wang & \quad University of California, Berkeley 
\\
Linfeng Li & \quad National University of Singapore 
\\
Shengqiong Wu & \quad University of Oxford 
\\
Leigang Qu & \quad National University of Singapore 
\\
Feng Li & \quad The Hong Kong University of Science and Technology 
\\
Wei Liu$^{\dagger}$ & \quad ByteDance Seed 
\\
Julian McAuley & \quad University of California, San Diego 
\\
Tat-Seng Chua & \quad National University of Singapore
\end{tabular}

\vspace{0.6em}

\noindent
$^{\dagger}$Corresponding author

\vspace{1em}

\section{Mathematical Foundations}
\label{app:math}

We derive the results of Section~\ref{sec:approach} under affine diffusion and flow parameterizations~\citep{song2020denoising,karras2022elucidating}.

\subsection{Parameterization and Objectives}
\label{app:local_action}

\noindent\textbf{Affine clean-output map.}
For an affine path \(z_t=\alpha_t y+\sigma_t\epsilon\) with \(v_t=\dot\alpha_t y+\dot\sigma_t\epsilon\) and \(\Delta_t=\alpha_t\dot\sigma_t-\dot\alpha_t\sigma_t\neq0\), the clean-output prediction is
\begin{equation}
    y_t(z_t,v_t)=\frac{\dot\sigma_t z_t-\sigma_t v_t}{\Delta_t}.
    \label{eq:app_general_action}
\end{equation}
At a fixed state, two velocity predictions satisfy \(\|y_v-y_{v^\star}\|_2=|\sigma_t/\Delta_t|\,\|v-v^\star\|_2\).
For rectified flow, \(\alpha_\sigma=1-\sigma\), \(\dot\alpha_\sigma=-1\), \(\dot\sigma_\sigma=1\), and \(\Delta_\sigma=1\), yielding \(y_\sigma=z_\sigma-\sigma v_\sigma\) and \(\delta y=-\sigma\delta v\) in Eq.~\eqref{eq:local_action}.
For \(\sigma>0\), this bijection scales velocity errors by \(\sigma\), so fixed clean-output radii induce larger velocity shifts at lower noise.

\phantomsection
\label{app:diffusionnft_relation}
\noindent\textbf{Endpoint-weighted vs.\ on-policy objectives.}
DiffusionNFT forms forward-noised states from sampled endpoints and optimizes
\begin{align*}
    \mathcal J_{\mathrm{NFT}}(\theta)
    =\mathbb E_{\substack{\mathbf c\sim\mathcal C,\;x_0\sim\pi_{\mathrm{old}}(\cdot\mid\mathbf c)\\
    \sigma\sim p_{\mathrm{train}},\;z_\sigma\sim p_{\mathrm{fwd}}(\cdot\mid x_0,\sigma)}}
    \left[\ell_\theta(z_\sigma,x_0,w(x_0,\mathbf c))\right].
\end{align*}
In \name{}, indexing rollout record \(b\) with query \(s_b\), anchor \(y_0(b)\), weight \(\omega_b\), and targets \(\bar y_\pm(b)\) yields
\begin{equation}
    \mathcal J_{\mathrm{OPSD}}(\theta)
    =\mathbb E_{b\sim Q_b}
    \left[\mathcal L_{\mathrm{OPSD}}
    \bigl(y_\theta(s_b),\bar y_+(b),\bar y_-(b),\omega_b\bigr)\right],
    \label{eq:app_opsd_population}
\end{equation}
where queries come directly from the EMA self-teacher rollout and targets are synthesized around \(y_0\) rather than copied from \(x_0\), as ablated in Section~\ref{sec:target_construction_results}.

\phantomsection
\label{app:reward_splits}
\phantomsection
\label{app:group_weights}
\noindent\textbf{Reward splits and group weights.}
Reward reweighting alone leaves sampled rollout outputs unchanged, whereas ascent and descent synthesis attaches an attractive target \(\bar y_+\) and a repulsive reference \(\bar y_-\) to each query.
Equation~\eqref{eq:group_weight} centers endpoint rewards within each prompt group, normalizes by the batch standard deviation, and clips \(\omega\) to \([0,1]\). Up to \(\epsilon_Z\), these weights are invariant to affine reward rescaling; when rewards within a group coincide, including \(K=1\), centered rewards vanish and \(\omega=1/2\), while the local reward gradient still supplies the synthesis direction.

\subsection{Self-Teacher Synthesis}
\label{app:trust_region}

\noindent\textbf{Bounded trust-region synthesis.}
Let $\widetilde{R}(y,\mathbf{c})=R(D(y),\mathbf{c})$ denote the decoded clean-output reward.
Equation~\eqref{eq:target_steps} forms a conservative clean-output update around $y_0$~\citep{kakade2002approximately,schulman2015trust}, analyzed via standard smoothness bounds~\citep{nocedal2006numerical,bubeck2015convex}.
Maximizing and minimizing $\widetilde R(y_0+\delta,\mathbf c)\approx\widetilde R(y_0,\mathbf c)+\langle g_0,\delta\rangle$ with $g_0=\nabla_y\widetilde R(y_0,\mathbf c)\neq 0$ over $\|\delta\|_2\le r_{\mathrm{tr}}$ yields
\begin{equation}
    \delta_+^\star=r_{\mathrm{tr}}\frac{g_0}{\|g_0\|_2},
    \qquad
    \delta_-^\star=-r_{\mathrm{tr}}\frac{g_0}{\|g_0\|_2}.
    \label{eq:app_trust_target}
\end{equation}
With $h_{\mathrm{step}}=\rho\|y_0\|_2/M_{\mathrm{tgt}}$ and nominal displacement $\bar h=M_{\mathrm{tgt}}h_{\mathrm{step}}=\rho\|y_0\|_2$, each stabilized step $h_{\mathrm{step}}g/(\|g\|_2+\epsilon_g)$ has norm below $h_{\mathrm{step}}$, so by the triangle inequality $\|y_\pm^{(M_{\mathrm{tgt}})}-y_0\|_2\le \rho\|y_0\|_2$.
Rescaling $R'(y)=cR(y)+d$ with $c>0$ gives $g'/(\|g'\|_2+\epsilon_g)=g/(\|g\|_2+\epsilon_g/c)$, preserving direction and length whenever $\|g\|_2\gg\epsilon_g/c$.

\phantomsection
\label{app:reward_improvement}
\noindent\textbf{Local reward improvement.}
We bound the local reward change of self-teacher synthesis around $y_0$.

\begin{opsdassumption}[Local smoothness]
\label{assump:smooth}
For a fixed prompt, $\widetilde{R}$ is differentiable and has $L$-Lipschitz gradients in a neighborhood of the self-teacher synthesis path.
\end{opsdassumption}

\begin{opsdlemma}[One-step ascent and descent]
\label{lemma:one_step_reward}
Let $g=\nabla_y\widetilde{R}(y)$ and $u_{\mathrm{grad}}=g/(\|g\|_2+\epsilon_g)$. Under Assumption~\ref{assump:smooth},
\begin{align*}
    \widetilde{R}(y+h_{\mathrm{step}} u_{\mathrm{grad}})-\widetilde{R}(y)
    &\geq
    h_{\mathrm{step}}\frac{\|g\|_2^2}{\|g\|_2+\epsilon_g}
    -\frac{Lh_{\mathrm{step}}^2}{2},\\
    \widetilde{R}(y-h_{\mathrm{step}} u_{\mathrm{grad}})-\widetilde{R}(y)
    &\leq
    -h_{\mathrm{step}}\frac{\|g\|_2^2}{\|g\|_2+\epsilon_g}
    +\frac{Lh_{\mathrm{step}}^2}{2}.
\end{align*}
\end{opsdlemma}

\begin{proof}
Applying the $L$-smoothness quadratic bounds at $\Delta=\pm h_{\mathrm{step}} u_{\mathrm{grad}}$ with $\langle g,h_{\mathrm{step}}u_{\mathrm{grad}}\rangle=h_{\mathrm{step}}\|g\|_2^2/(\|g\|_2+\epsilon_g)$ and $\|h_{\mathrm{step}}u_{\mathrm{grad}}\|_2\le h_{\mathrm{step}}$ yields both inequalities.
\end{proof}

\begin{opsdprop}[Multi-step target bound]
\label{prop:multi_step_bound}
Let $y_+^{(0)}=y_0$ and $y_+^{(m+1)}=y_+^{(m)}+h_{\mathrm{step}} g_m/(\|g_m\|_2+\epsilon_g)$, where $g_m=\nabla_y\widetilde{R}(y_+^{(m)})$. If all steps remain in the smooth neighborhood, then
\begin{equation}
    \widetilde{R}(y_+^{(M_{\mathrm{tgt}})})-\widetilde{R}(y_0)
    \geq
    \sum_{m=0}^{M_{\mathrm{tgt}}-1}
    h_{\mathrm{step}}\frac{\|g_m\|_2^2}{\|g_m\|_2+\epsilon_g}
    -\frac{L M_{\mathrm{tgt}} h_{\mathrm{step}}^2}{2}.
    \label{eq:app_multistep_bound}
\end{equation}
An analogous upper bound holds for the negative target path.
\end{opsdprop}

\begin{proof}
Applying Lemma~\ref{lemma:one_step_reward} at each step $m=0,\ldots,M_{\mathrm{tgt}}-1$ and telescoping the sum yields Eq.~\eqref{eq:app_multistep_bound}.
\end{proof}

With $h_{\mathrm{step}}=\rho\|y_0\|_2/M_{\mathrm{tgt}}$, the curvature penalty equals $L\rho^2\|y_0\|_2^2/(2M_{\mathrm{tgt}})$, decreasing inversely with $M_{\mathrm{tgt}}$ at fixed $\rho$.
For $g\neq 0$, the one-step lower bound in Lemma~\ref{lemma:one_step_reward} is strictly positive when
\begin{equation}
    0<h_{\mathrm{step}}<
    \frac{2\|g\|_2^2}
    {L(\|g\|_2+\epsilon_g)}.
\label{eq:app_one_step_size_chain}
\end{equation}

\phantomsection
\label{app:surrogate_validity}
\noindent\textbf{Fixed-suffix reward alignment.}
Inserting clean-output prediction $y$ at query $s$ sets $v_y=(z_q-y)/\sigma_q$ in $z_{q-1}=\Phi_q(z_q;v_y,\mathbf c)$. Continuing the remaining steps with the frozen EMA self-teacher yields endpoint $S_{q\rightarrow0}^{\mathrm{old}}(y,z_q,\mathbf c)$ and fixed-suffix reward $F_q(y)=R(D(S_{q\rightarrow0}^{\mathrm{old}}(y,z_q,\mathbf c)),\mathbf c)$.

\begin{opsdassumption}[Fixed-suffix reward regularity]
\label{assump:suffix_regular}
In a neighborhood of $y_0$, $F_q$ is $L_F$-smooth, and $g_F=\nabla_yF_q(y_0,z_q,\mathbf{c})$ satisfies $\langle g_F,g\rangle\geq \kappa\|g_F\|_2\|g\|_2$ with $g=\nabla_y\widetilde{R}(y_0,\mathbf{c})$ for some $\kappa>0$.
\end{opsdassumption}

\begin{opsdprop}[Local reward ascent improves fixed-suffix reward]
\label{prop:surrogate_suffix}
Under Assumption~\ref{assump:suffix_regular}, let $u_{\mathrm{grad}}=g/(\|g\|_2+\epsilon_g)$ and $y_+=y_0+h_{\mathrm{suffix}} u_{\mathrm{grad}}$. Then
\begin{equation}
    F_q(y_+)-F_q(y_0)
    \geq
    h_{\mathrm{suffix}}\kappa\frac{\|g_F\|_2\|g\|_2}{\|g\|_2+\epsilon_g}
    -\frac{L_Fh_{\mathrm{suffix}}^2}{2}.
    \label{eq:app_suffix_improvement}
\end{equation}
\end{opsdprop}

\begin{proof}
Smoothness of $F_q$ yields $F_q(y_+)\ge F_q(y_0)+h_{\mathrm{suffix}}\langle g_F,u_{\mathrm{grad}}\rangle-\frac{L_F}{2}h_{\mathrm{suffix}}^2\|u_{\mathrm{grad}}\|_2^2$, and substituting $\langle g_F,u_{\mathrm{grad}}\rangle\ge\kappa\|g_F\|_2\|g\|_2/(\|g\|_2+\epsilon_g)$ with $\|u_{\mathrm{grad}}\|_2\le 1$ proves Eq.~\eqref{eq:app_suffix_improvement}.
\end{proof}

For $L_R$-Lipschitz reward $R$ and remaining sampler map satisfying $\|D(S_{q\rightarrow0}^{\mathrm{old}}(y,z_q,\mathbf c))-D(y)\|_2\le C_0\sigma_q+C_1\|y-y_0\|_2$,
\begin{equation}
    |F_q(y,z_q,\mathbf{c})-\widetilde{R}(y,\mathbf{c})|
    \leq L_R(C_0\sigma_q+C_1\|y-y_0\|_2),
\label{eq:app_surrogate_composition_chain}
\end{equation}
showing that $\widetilde R$ tracks $F_q$ tightly when both $\sigma_q$ and $\|y-y_0\|_2$ are small.

\subsection{Distillation Objective}
\label{app:implicit_branch}

\noindent\textbf{Branch geometry and ideal output improvement.}
Treat the adaptive normalizers as detached constants and set $a_+=\omega/\gamma_+$, $a_-=(1-\omega)/\gamma_-$, $\delta_\theta=y_\theta-y_0$, $d_+=\bar y_+-y_0$, and $d_-=\bar y_- -y_0$.
Up to $1/d_y$ over dimension $d_y$, Eq.~\eqref{eq:branches} gives
\begin{equation}
\begin{aligned}
    \mathcal L_{\mathrm{OPSD}}(\delta_\theta)
    &=a_+\|\beta\delta_\theta-d_+\|_2^2+a_-\|\beta\delta_\theta+d_-\|_2^2\\
    &=\beta^2(a_++a_-)\left\|\delta_\theta-\frac{F}{\beta(a_++a_-)}\right\|_2^2
      +C-\frac{\|F\|_2^2}{a_++a_-},
\end{aligned}
\label{eq:app_branch_square_chain}
\end{equation}
with \(F=a_+ d_+-a_- d_-\) and \(C=a_+\|d_+\|_2^2+a_-\|d_-\|_2^2\). For \(a_++a_->0\) and \(\beta>0\), the minimizer is
\begin{equation}
    \delta_\theta^\star=\frac{a_+ d_+-a_- d_-}{\beta(a_++a_-)},
    \label{eq:app_branch_optimum}
\end{equation}
which reduces to \(\delta_\theta^\star=\bar h u_{\mathrm{grad}}/\beta\) in the symmetric case \(d_\pm=\pm\bar h u_{\mathrm{grad}}\).

\phantomsection
\label{app:target_asymmetry}
When \(M_{\mathrm{tgt}}\ge 2\), positive and negative paths capture curvature terms \(c_+,c_-\) via \(d_+=\bar h u_{\mathrm{grad}}+c_+\) and \(d_-=-\bar h u_{\mathrm{grad}}+c_-\), yielding
\begin{equation}
    \delta_\theta^\star=\frac{\bar h}{\beta}u_{\mathrm{grad}}+c_{\mathrm{br}},
    \qquad
    c_{\mathrm{br}}=\frac{a_+ c_+-a_- c_-}{\beta(a_++a_-)},
    \qquad
    \|c_{\mathrm{br}}\|_2\leq\frac{a_+\|c_+\|_2+a_-\|c_-\|_2}{\beta(a_++a_-)}.
    \label{eq:app_asymmetric_optimum}
\end{equation}

\phantomsection
\label{app:policy_improvement}
Under Assumption~\ref{assump:smooth} with \(g=\nabla_y\widetilde R(y_0,\mathbf c)\), the ideal output \(y_\theta^\star=y_0+\delta_\theta^\star\) satisfies
\begin{equation}
    \widetilde R(y_\theta^\star)-\widetilde R(y_0)
    \geq\langle g,\delta_\theta^\star\rangle-\frac L2\|\delta_\theta^\star\|_2^2,
    \label{eq:app_general_policy_improvement}
\end{equation}
which is positive whenever \(\langle g,a_+ d_+-a_- d_-\rangle>\frac{L}{2\beta(a_++a_-)}\|a_+ d_+-a_- d_-\|_2^2\), and simplifies for symmetric targets to
\begin{equation}
    \widetilde R(y_\theta^\star)-\widetilde R(y_0)
    \geq\frac{\bar h}{\beta}\frac{\|g\|_2^2}{\|g\|_2+\epsilon_g}
    -\frac{L\bar h^2}{2\beta^2}.
    \label{eq:app_policy_improvement}
\end{equation}

\phantomsection
\label{app:gradient_direction}
\phantomsection
\label{app:velocity_mse_equivalence}
\noindent\textbf{Velocity-space equivalence.}
At a fixed query, Eq.~\eqref{eq:app_general_action} yields \(\nabla_{v_t}\mathcal L=-\frac{\sigma_t}{\Delta_t}\nabla_{y_t}\mathcal L\), or \(-\sigma_q\nabla_y\mathcal L\) under rectified flow.
With \(\bar v_\pm=(z_q-\bar y_\pm)/\sigma_q\) and \(v_\theta^\pm=(1\mp\beta)v_{\mathrm{old}}\pm\beta v_\theta\), we have \(y_\theta^\pm-\bar y_\pm=-\sigma_q(v_\theta^\pm-\bar v_\pm)\), so Eq.~\eqref{eq:opsd_loss} is equivalent to
\begin{equation}
    \mathcal L_{\mathrm{OPSD}}
    =\omega\frac{\sigma_q^2}{\gamma_+}\operatorname{mean}[(v_\theta^+-\bar v_+)^2]
    +(1-\omega)\frac{\sigma_q^2}{\gamma_-}\operatorname{mean}[(v_\theta^- -\bar v_-)^2],
    \label{eq:app_velocity_branch_loss}
\end{equation}
with ideal velocity update \(\Delta v_\theta^\star=-\delta_\theta^\star/\sigma_q\). Unlike DanceOPD and DiffusionOPD, which match frozen external teachers, \name{} derives \(\bar v_\pm\) from reward-guided self-teachers around \(v_{\mathrm{old}}\).

\phantomsection
\label{app:adaptive_scaling}
\noindent\textbf{Adaptive scaling.}
For residual \(r\in\mathbb R^{d_y}\) and \(m(r)=\max\{\operatorname{mean}|r|,\epsilon_\gamma\}\), the normalized loss \(\ell(r)=\frac{\|r\|_2^2/d_y}{\operatorname{sg}(m(r))}\) has gradient \(\nabla_r\ell=\frac{2r}{d_y\,m(r)}\), which is scale-invariant above \(\epsilon_\gamma\) while preserving direction. With detached \(\gamma_+,\gamma_-\) and \(r_\pm=y_\theta^\pm-\bar y_\pm\),
\begin{equation}
    \nabla_{y_\theta}\mathcal L_{\mathrm{OPSD}}
    =\frac{2\beta}{d_y}\left[
    \frac{\omega r_+}{\gamma_+}
    -\frac{(1-\omega)r_-}{\gamma_-}\right].
    \label{eq:app_full_adaptive_gradient}
\end{equation}

\subsection{Realization and Cost}
\label{app:detached_objective}
\label{app:target_update_gap}

\noindent\textbf{Synthesis and realized gains.}
For fixed query $s$, anchor $y_0$, positive self-teacher $\bar y_+$, and post-distillation student prediction $\hat y_{M_{\mathrm{dis}}}$, the synthesis gain $\Delta F_q(\bar y_+)=F_q(\bar y_+)-F_q(y_0)$ measures intrinsic target improvement prior to parameter updates, whereas the realized gain $\Delta F_q(\hat y_{M_{\mathrm{dis}}})=F_q(\hat y_{M_{\mathrm{dis}}})-F_q(y_0)$ measures what survives after $M_{\mathrm{dis}}$ updates.
Under a local preconditioned step with Jacobian $J_i=\partial y_\theta(s_i)/\partial\theta$, target displacements $d_i=\bar y_{+,i}-y_{0,i}$ induce the kernel-filtered response $\Delta y(s)\approx\eta_{\mathrm{dis}}\sum_i\mathcal K(s,s_i)d_i$ with $\mathcal K(s,s_i)=J_s\mathsf P J_i^\top$, requiring separate synthesis and realization audits.

\noindent\textbf{Single-branch self-distillation dynamics.}
In the single-step symmetric setting $M_{\mathrm{tgt}}=1$ of Eq.~\eqref{eq:opsd_star_reduction}, with $d=\rho\|y_0\|_2 u_{\mathrm{grad}}$ and $J=\partial y_\theta(s)/\partial\theta$, we have $\nabla_\theta\mathcal L_{\mathrm{OPSD}^\ast}(\theta)=\frac{2\beta}{d_y\gamma}J^\top(\beta(y_\theta(s)-y_0)-d)$.
At $\theta=\theta_{\mathrm{old}}$ with $\gamma=\operatorname{mean}|d|>\epsilon_\gamma$, the radius $\rho\|y_0\|_2$ cancels, yielding $-\frac{2\beta}{d_y\operatorname{mean}|u_{\mathrm{grad}}|}J^\top u_{\mathrm{grad}}$ without retaining decoder or reward graphs.
Once $\theta\neq\theta_{\mathrm{old}}$ or $M_{\mathrm{tgt}}\ge 2$, online self-distillation stabilizes updates beyond normalized backpropagation by evaluating $u_{\mathrm{grad}}(y_0)$ at the stationary EMA anchor with restoring term $\frac{2\beta^2}{d_y\gamma}J^\top(y_\theta(s)-y_0)$, adapting $\gamma$ to student-teacher drift, and capturing multi-step curvature asymmetry.

\phantomsection
\label{app:computational_cost}
\noindent\textbf{Computational cost.}
For \(N_B\) queries, \(M_{\mathrm{tgt}}\) synthesis steps per branch, and \(M_{\mathrm{dis}}\) updates per iteration, the leading-order cost is
\begin{equation}
    C_{\mathrm{OPSD}}
    \approx
    N_B C_{\mathrm{roll}}
    +N_B(2M_{\mathrm{tgt}}-1)C_{D+R}^{\nabla}
    +N_B M_{\mathrm{dis}} C_{\mathrm{pol}},
\label{eq:app_opsd_cost}
\end{equation}
where \(C_{\mathrm{roll}}\), \(C_{D+R}^{\nabla}\), and \(C_{\mathrm{pol}}\) denote amortized per-query rollout, decoder-reward backward, and student update costs. Sharing the anchor gradient at \(y_0\) requires \(2M_{\mathrm{tgt}}-1\) reward backward calls per query, while detached state storage avoids retaining activations across denoising steps.

\section{Experiment Details}
\label{app:experiment_details}

\subsection{Main Results}
\label{app:main_result_details}

\noindent\textbf{Main-table scope.}
Table~\ref{tab:main_results} evaluates held-out DrawBench generations across ten text-to-image evaluators.
Rows marked with \cmark\ train a dedicated $100$-update specialist per evaluator column, whereas rows marked with \xmark\ evaluate one checkpoint across all ten columns, covering reference baselines, $300$-update joint policies trained on PickScore, CLIPScore, and HPSv$2.1$, and two-stage distillation baselines. Public reference rows marked with $^\ddagger$ retain their released configurations. \name{}$^\ast$, abbreviated OPSD$^\ast$ in figures, uses single-trajectory collection with $K=1$ and $\omega=1/2$, single-step synthesis with $M_{\mathrm{tgt}}=1$, and query-prefix truncation at $\sigma_q$.

\noindent\textbf{Backbone and baseline protocols.}
All locally trained SD$3.5$-M rows fine-tune LoRA adapters~\citep{hu2021lora} of rank $32$ and alpha $64$ at $512\times 512$ resolution on Pick-a-Pic prompts using $48$ prompt groups of $24$ trajectories per update, or $1{,}152$ single-trajectory prompts for \name{}$^\ast$. Held-out evaluation uses $200$ DrawBench prompts with deterministic $40$-step flow sampling at guidance scale $1.0$ and five fixed seeds per prompt. Z-Image-Turbo rows adapt the released $9$-step distilled checkpoint~\citep{liu2025decoupleddmdcfgaugmentation,jiang2025distribution} at $1024\times 1024$ resolution without teacher access, using the native deterministic $9$-step Euler schedule and $48$ prompt groups of $12$ trajectories per update, or $576$ single-trajectory prompts for \name{}$^\ast$. DiffusionNFT, ReFL, and \name{} use CFG-free deterministic rollouts, whereas FlowGRPO uses stochastic SDE-flow rollouts at $\eta_{\mathrm{SDE}}=0.7$ for per-transition likelihood ratios. DiffusionNFT retains its official frozen-base velocity regularization, ReFL backpropagates late-step clean-output rewards on matched prompt streams with prefix truncation, and DRaFT-$1$~\citep{clark2024directly} detaches the initial denoising prefix and backpropagates through the final step, decoder, and all three joint rewards.

\phantomsection
\label{app:opd_baselines}
\noindent\textbf{Two-stage on-policy distillation baselines.}
For the $^\dagger$ rows in Table~\ref{tab:main_results}, we freeze three $100$-update \name{} specialists trained on PickScore, CLIPScore, and HPSv$2.1$ as equal-weight teachers for DanceOPD, DiffusionOPD, and FlowOPD, while recent OPD variants also supervise intermediate features~\citep{stepopd2026,focusopd2026}.
Under the shared SD$3.5$-M LoRA setup and matched second-stage wall-clock budget, DanceOPD matches teacher velocities at a low-noise student query, DiffusionOPD matches teacher transition means across all $10$ rollout steps, and FlowOPD applies PPO-clipped updates~\citep{schulman2017proximal} on stochastic transitions scored under teacher densities.

\noindent\textbf{Optimization and reproducibility.}
All \name{} runs use AdamW with learning rate $3\times 10^{-4}$, weight decay $10^{-4}$, self-teacher EMA retention $\eta_{\mathrm{ema}}^{(u)}=\min\{\frac{u}{u+10},0.5\}$ with $M_{\mathrm{dis}}=1$, and checkpoint EMA retention $\eta_{\mathrm{ckpt}}^{(u)}=\min\{\frac{u+1}{u+10},0.9\}$. We set $M_{\mathrm{tgt}}=2$, $\rho=0.10$, $c_{\mathrm{adv}}=5$, and stabilizers $\epsilon_Z=10^{-4}$, $\epsilon_g=10^{-12}$, and $\epsilon_\gamma=10^{-5}$, with requested query levels $\sigma^\star=0.278$ on SD$3.5$-M and $\sigma^\star=0.273$ on Z-Image-Turbo. Locally trained baselines select learning rates on a disjoint Pick-a-Pic validation split with their official regularizers, gradient clipping, and checkpoint EMA enabled.

\noindent\textbf{Evaluator definitions and cross-reward transfer.}
Public evaluators comprise PickScore v$1$, CLIPScore~\citep{hessel2021clipscore}, HPSv$2.1$~\citep{wu2023human}, LAION Aesthetic~\citep{schuhmann2022laion}, ImageReward v$1.0$, HPSv$3$~\citep{ma2025hpsv3}, and DeQA-Score-Mix$3$~\citep{you2025teaching}. Internal evaluators include AltCLIP, an alignment model trained on internal preference data following the AltCLIP architecture~\citep{chen2023altclip}, VLM-Pointwise, an internal vision-language preference model scoring prompt alignment and perceptual quality, and VLM-Pairwise, which compares generations against fixed Seedream $5.0$ Pro~\citep{seed2026seedream5pro} references across both presentation orders. On held-out non-target metrics, $100$-update SD$3.5$-M specialists trained on CLIPScore, HPSv$2.1$, or PickScore improve all three non-target evaluators by up to $25.1\%$ on HPSv$2.1$, $11.5\%$ on PickScore, and $21.3\%$ on Aesthetic, surpassing ReFL and DiffusionNFT.

\subsection{Efficiency and Dynamics}
\label{app:efficiency_profile}

\begin{table}[t]
\centering
\vspace{-0.3cm}
\scriptsize
\setlength{\tabcolsep}{2.4pt}
\caption{\textbf{Training efficiency and iso-quality profiling.} Per-step wall time, per-rank evaluator and diffusion backward counts, throughput, total compute, and compute to match baseline terminal quality across backbones.}
\label{tab:efficiency_profile}
\resizebox{\linewidth}{!}{%
\begin{tabular}{llcccccccc}
\toprule
\textbf{Backbone} & \textbf{Method} & \textbf{Time} & \textbf{Reward fwd;} & \textbf{Diff.} & \textbf{Images} & \textbf{GPU-h} & \textbf{GPU-h to} & \textbf{GPU-h to} & \textbf{Final} \\
& & \textbf{s/step} & \textbf{teacher grad} & \textbf{bwd} & \textbf{per s} & \textbf{per 100} & \textbf{NFT@100} & \textbf{ReFL@100} & \textbf{score} \\
\midrule
& DiffusionNFT & $212.4\pm0.4$ & $16\,;\,0$ & $144$ & $5.42$ & $47.2$ & $47.1$ & never & $0.298$ \\
& ReFL & $214.8$ & $144\,;\,0$ & $144$ & $5.36$ & $47.7$ & $19.1$ & $47.7$ & $0.308$ \\
& FlowGRPO & $191.1$ & $16\,;\,0$ & $160$ & $6.03$ & $42.5$ & never & never & $0.263$ \\
\rowcolor{vefoursrow}\cellcolor{white}
& \name{} & $126.9\pm0.2$ & $88\,;\,72$ & $16$ & $9.08$ & $28.2$ & $6.6$ & $13.2$ & \textbf{0.340} \\
\rowcolor{vefoursrow}\cellcolor{white}\multirow{-5}{*}{SD$3.5$-M}
& \name{}$^\ast$ & $80.4\pm0.2$ & $24\,;\,24$ & $24$ & \textbf{14.33} & \textbf{17.9} & \textbf{5.4} & \textbf{7.1} & $0.327$ \\
\midrule
& DiffusionNFT & $1826.2\pm0.7$ & $12\,;\,0$ & $96$ & $0.32$ & $405.8$ & $0.0$ & never & $0.277$ \\
& ReFL & $459.5$ & $72\,;\,0$ & $72$ & \textbf{1.25} & \textbf{102.1} & $0.0$ & $102.1$ & $0.380$ \\
& FlowGRPO & $1475.2\pm2.1$ & $12\,;\,0$ & $108$ & $0.39$ & $327.8$ & $0.0$ & never & $0.305$ \\
\rowcolor{vefoursrow}\cellcolor{white}
& \name{} & $674.0\pm0.6$ & $120\,;\,108$ & $12$ & $0.85$ & $149.8$ & $0.0$ & $115.0$ & \textbf{0.390} \\
\rowcolor{vefoursrow}\cellcolor{white}\multirow{-5}{*}{Z-Image-Turbo}
& \name{}$^\ast$ & $480.2\pm1.0$ & $36\,;\,36$ & $8$ & $1.20$ & $106.7$ & $0.0$ & \textbf{96.0} & $0.386$ \\
\bottomrule
\end{tabular}%
}
\par\vspace{3pt}
\end{table}

\noindent\textbf{Efficiency profiling.}
Table~\ref{tab:efficiency_profile} reports matched-budget online efficiency after warmup, covering rollout, decoding, reward evaluation, synthesis, and updates while excluding initialization and offline evaluation. We profile SD$3.5$-M CLIPScore at $512\times 512$ with $1{,}152$ trajectories per update and Z-Image-Turbo HPSv$2.1$ at $1024\times 1024$ with $576$ trajectories per update, alongside GPU-hours needed to match the $100$-update terminal scores of DiffusionNFT and ReFL.
DiffusionNFT and FlowGRPO backpropagate across denoising transitions without reward gradients, whereas ReFL retains coupled graphs during backward passes, yielding fast steps on $9$-step Z-Image-Turbo with lightweight HPSv$2.1$ relative to two-step synthesis in \name{}.
Detaching single-step synthesis in \name{}$^\ast$ matches ReFL throughput, reaches its terminal score in fewer GPU-hours, and cuts diffusion backward calls by up to $88.9\%$.

\phantomsection
\label{app:native_dynamics}
\noindent\textbf{Training reward curves.}
Figures~\ref{fig:main_table_native_dynamics},~\ref{fig:native_gpu_pareto_20reward}, and~\ref{fig:two_backbone_ten_reward_pareto} track cumulative training compute across all twenty single-reward settings, using measured end-to-end step durations for local evaluators and policy-side compute for asynchronously served evaluators.
Figures~\ref{fig:native_dynamics_sd35m} and~\ref{fig:native_dynamics_zimage} summarize normalized training progress, terminal retention, and multi-seed held-out stability by backbone, where SD$3.5$-M FlowGRPO logs native CLIPScore trajectories and is evaluated across all ten held-out benchmarks in Figure~\ref{fig:two_backbone_ten_reward_pareto}. Within each setting and the multi-reward panel of Figure~\ref{fig:main_table_native_dynamics}, curves normalize progress by the span between initial reward and peak smoothed gain, displaying medians and interquartile ranges across objectives.

\begin{figure}[t]
    \centering
    \vspace{-0.4cm}
    \includegraphics[width=\linewidth]{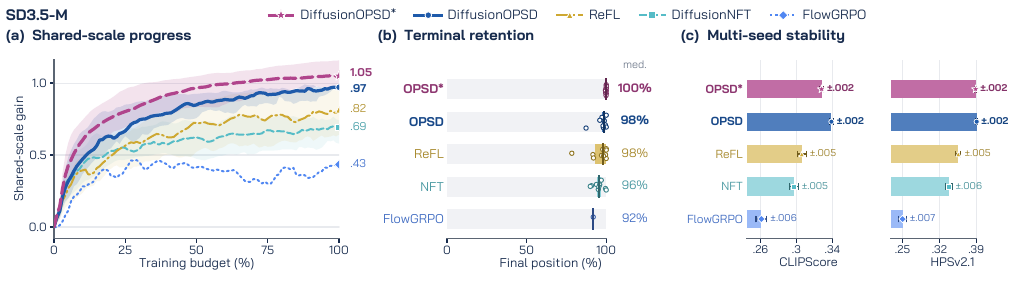}
    \vspace{-0.6cm}
    \caption{\textbf{Training dynamics and multi-seed stability on the multi-step backbone.} (a) Shared-scale training progress, (b) terminal retention, and (c) held-out means and cross-seed standard deviations across independent seeds.}
    \label{fig:native_dynamics_sd35m}
    \vspace{-0.2cm}
\end{figure}

\begin{figure}[t]
    \centering
    \vspace{-0.4cm}
    \includegraphics[width=\linewidth]{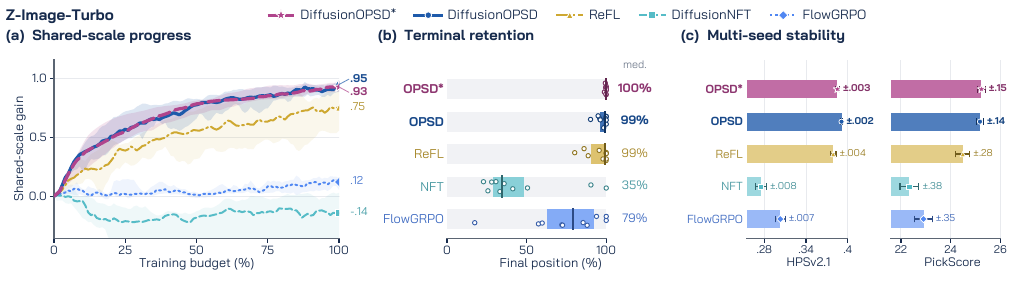}
    \vspace{-0.6cm}
    \caption{\textbf{Training dynamics and multi-seed stability on the step-distilled backbone.} (a) Shared-scale training progress, (b) terminal retention, and (c) held-out means and cross-seed standard deviations across independent seeds.}
    \label{fig:native_dynamics_zimage}
    \vspace{-0.2cm}
\end{figure}

\subsection{Mechanism Audits}
\label{app:finite_fitting_audit}

\begin{table}[!hb]
\centering
\vspace{-0.3cm}
\scriptsize
\setlength{\tabcolsep}{4.2pt}
\caption{\textbf{Same-query coupling and rollout verification audit.} Paired single-update comparison at shared query states and self-teacher anchors under matched pooled displacement, reporting memory, gradient dispersion, off-query drift, fixed-suffix improvement rate, and regenerated rollout reward gain.}
\label{tab:post_fitting_audit}
\resizebox{\linewidth}{!}{%
\begin{tabular}{lccccccc}
\toprule
\textbf{Interface} & \textbf{Teacher rel. norm} & \textbf{Peak VRAM} & \textbf{Grad norm CV} & \textbf{Drift / shift} & \textbf{Off-query drift} & \textbf{$P(\Delta F_q>0)$} & \textbf{Full-rollout $\Delta R_{\mathrm{roll}}$} \\
\midrule
DiffusionNFT endpoint & $0.247$ & $31.8\text{ GiB}$ & $0.208$ & $0.649$ & $2.9\times 10^{-4}$ & $41.8\%$ & $-7.1\times 10^{-3}$ \\
Direct backprop ReFL & - & $37.6\text{ GiB}$ & $0.399$ & $0.388$ & $1.4\times 10^{-4}$ & $93.0\%$ & - \\
Detached unbounded & unbounded & $31.8\text{ GiB}$ & $0.476$ & \textbf{0.369} & \textbf{1.3}$\boldsymbol{\times}$\textbf{10}$^{\boldsymbol{-4}}$ & $89.3\%$ & - \\
\rowcolor{vefoursrow}
\name{} self-teacher & \textbf{0.071} & \textbf{31.8 GiB} & \textbf{0.262} & $0.371$ & $1.4\times 10^{-4}$ & \textbf{94.3\%} & \textbf{+4.5}$\boldsymbol{\times}$\textbf{10}$^{\boldsymbol{-3}}$ \\
\bottomrule
\end{tabular}%
}
\end{table}

In the same-query and post-distillation audits of Section~\ref{sec:target_construction_results} and Figure~\ref{fig:ablation_efficiency}(e), (f), (g), and (h), we evaluate $512$ held-out prompts at $\rho=0.10$ using update-$50$ policies at $\sigma_q\approx 0.28$ on SD$3.5$-M and $\sigma_q\approx 0.27$ on Z-Image-Turbo.
At default low-noise queries, the fixed-suffix gradient $g_F=\nabla_y F_q(y_0)$ and local reward gradient $g=\nabla_y\widetilde R(y_0,\mathbf c)$ reach mean cosine similarities above $0.64$ with positive inner-product rates above $91.2\%$ across CLIPScore, HPSv$2.1$, and PickScore on SD$3.5$-M, and $0.76$ cosine similarity with $96.1\%$ positive alignment on Z-Image-Turbo; raising noise to $\sigma_q=0.90$ drops cosine similarity to $0.14$.

Table~\ref{tab:post_fitting_audit} and Figure~\ref{fig:ablation_efficiency}(g) and (h) audit pre-update shifts, single-step realized shifts, and regenerated $40$-step rollout rewards.
Synthesized self-teachers stay strictly inside $\rho=0.10$ with $0.0\%$ violations and positive fixed-suffix gains, whereas $100.0\%$ of DiffusionNFT endpoint displacements exceed $\rho=0.10$ and incur suffix losses.
After one AdamW update, realized relative shifts remain conservative with $95$th percentiles at most $0.018$.
Under matched pooled displacement in Table~\ref{tab:post_fitting_audit}, \name{} cuts off-query drift by $51.7\%$, raises the positive fixed-suffix rate by $52.5$ percentage points over DiffusionNFT, and yields positive regenerated rollout gains on $93.9\%$ of prompts.

Figure~\ref{fig:ablation_efficiency}(e) and (f) trace four matched-budget $50$-update stages across SD$3.5$-M CLIPScore, Z-Image-Turbo HPSv$2.1$, SD$3.5$-M ImageReward, and Z-Image-Turbo HPSv$3$ at $K=1$ and $\sigma_q$ truncation, tracking gradient-norm dispersion via the coefficient of variation $\mathrm{CV}$.
Per-sample $\ell_2$ normalization with stop-gradient detachment improves held-out scores over coupled ReFL by up to $9.8\%$ while reducing peak VRAM by $15.4\%$ to match the pure diffusion footprint.
Adding stationary EMA self-teacher anchoring lifts cumulative gains to $32.7\%$, and adding adaptive residual scaling in \name{}$^\ast$ reaches up to $36.5\%$ over coupled ReFL and $24.3\%$ over normalized ReFL while cutting $\mathrm{CV}$ by up to $50.5\%$.

\subsection{Ablations and CFG}
\label{app:ablation_details}

\noindent\textbf{Protocol and sensitivity grid.}
Ablations in Section~\ref{sec:ablations} train rank-$32$ LoRA adapters on SD$3.5$-M CLIPScore for $50$ updates with deterministic $10$-step DPM-Solver++ $2$M rollouts~\citep{lu2025dpm} and evaluate on DrawBench with deterministic $40$-step flow sampling and $10{,}000$-resample bootstrap intervals.
The canonical reference uses $\sigma_q\approx 0.28$, $\rho=0.10$, $M_{\mathrm{tgt}}=2$, $\beta=1.0$, and self-teacher EMA, while Figure~\ref{fig:ablation_robustness}(a) varies training and inference guidance across $\{1, 4.5, 6\}$.
Figures~\ref{fig:ablation_mechanism_qual},~\ref{fig:ablation_progress_mechanism_p02},~\ref{fig:ablation_progress_mechanism_p03},~\ref{fig:ablation_direction_qual}, and~\ref{fig:ablation_progress_direction} show qualitative comparisons across query sources, branch scaling, and synthesis direction controls.

\noindent\textbf{Human preference and compositional evaluation.}
In Figure~\ref{fig:ablation_robustness}(b) and (c), $12$ STEM-educated annotators provide $2{,}400$ blinded pairwise judgments across $100$ Z-Image-Turbo VLM-Pointwise prompts and $100$ SD$3.5$-M HPSv$2.1$ prompts with Fleiss' $\kappa=0.62$.
Under half-tie scoring with $10{,}000$-resample bootstrap intervals, \name{} reaches win rates up to $90\%$ over DiffusionNFT, $74\%$ over FlowGRPO, and $63\%$ over ReFL.
In Figure~\ref{fig:ablation_robustness}(d), $1{,}500$ blinded ratings per axis with $\kappa=0.64$ compare jointly trained SD$3.5$-M \name{} against guided SD$3.5$-M, DanceOPD, DiffusionNFT, DRaFT, and ReFL, reaching win rates up to $78\%$ on prompt fidelity, $77\%$ on image quality, and $81\%$ on realism.
Across DrawBench categories in Figure~\ref{fig:ablation_robustness}(e), the same checkpoint outperforms DRaFT by up to $13.9\%$, DiffusionNFT by up to $17.3\%$, and ReFL by up to $33.7\%$ in HPSv$2.1$ while lifting overall unguided GenEval accuracy by $117.9\%$.

\phantomsection
\label{app:cfg_compensation}
\noindent\textbf{CFG compensation.}
At a fixed query, classifier-free guidance with scale \(g_{\mathrm{cfg}}\) forms \(\mathcal A_{g_{\mathrm{cfg}}}(v_c,v_u)=g_{\mathrm{cfg}} v_c+(1-g_{\mathrm{cfg}})v_u\).
In guided training at scale \(g_{\mathrm{tr}}\), let \(\Delta v_c,\Delta v_u\) denote conditional and unconditional student changes relative to the EMA self-teacher, let \(\Delta v_{\mathrm{tar}}\) denote the self-teacher velocity correction, and define \(e_{\mathrm{tr}}=\mathcal A_{g_{\mathrm{tr}}}(\Delta v_c,\Delta v_u)-\Delta v_{\mathrm{tar}}\) and \(\Delta v_{\mathrm{gap}}=\Delta v_c-\Delta v_u\).
Guided distillation constrains only one affine branch combination rather than each branch individually.

\begin{opsdprop}[CFG compensation identity]
\label{prop:cfg_compensation}
At evaluation guidance scale \(g_{\mathrm{ev}}\),
\begin{equation}
    \mathcal A_{g_{\mathrm{ev}}}(\Delta v_c,\Delta v_u)
    -\Delta v_{\mathrm{tar}}
    =e_{\mathrm{tr}}+(g_{\mathrm{ev}}-g_{\mathrm{tr}})\Delta v_{\mathrm{gap}}.
    \label{eq:cfg_compensation_identity}
\end{equation}
If \(e_{\mathrm{tr}}=0\), the squared field error is \(\left\|\mathcal A_{g_{\mathrm{ev}}}(\Delta v_c,\Delta v_u)-\Delta v_{\mathrm{tar}}\right\|_2^2=(g_{\mathrm{ev}}-g_{\mathrm{tr}})^2\|\Delta v_{\mathrm{gap}}\|_2^2\).
\end{opsdprop}
\begin{proof}
Subtracting the train-scale combination from the evaluation-scale combination yields \((g_{\mathrm{ev}}-g_{\mathrm{tr}})(\Delta v_c-\Delta v_u)\), from which adding \(e_{\mathrm{tr}}\) and setting \(e_{\mathrm{tr}}=0\) give both identities.
\end{proof}
Evaluating at \(g_{\mathrm{ev}}\neq g_{\mathrm{tr}}\) exposes the unconstrained branch gap \(\Delta v_{\mathrm{gap}}\), whereas frozen-teacher CFG absorption in DanceOPD fixes a guided teacher field before fitting one student field.

\begin{figure}[p]
    \centering
    \includegraphics[width=\linewidth,height=0.96\textheight,keepaspectratio]{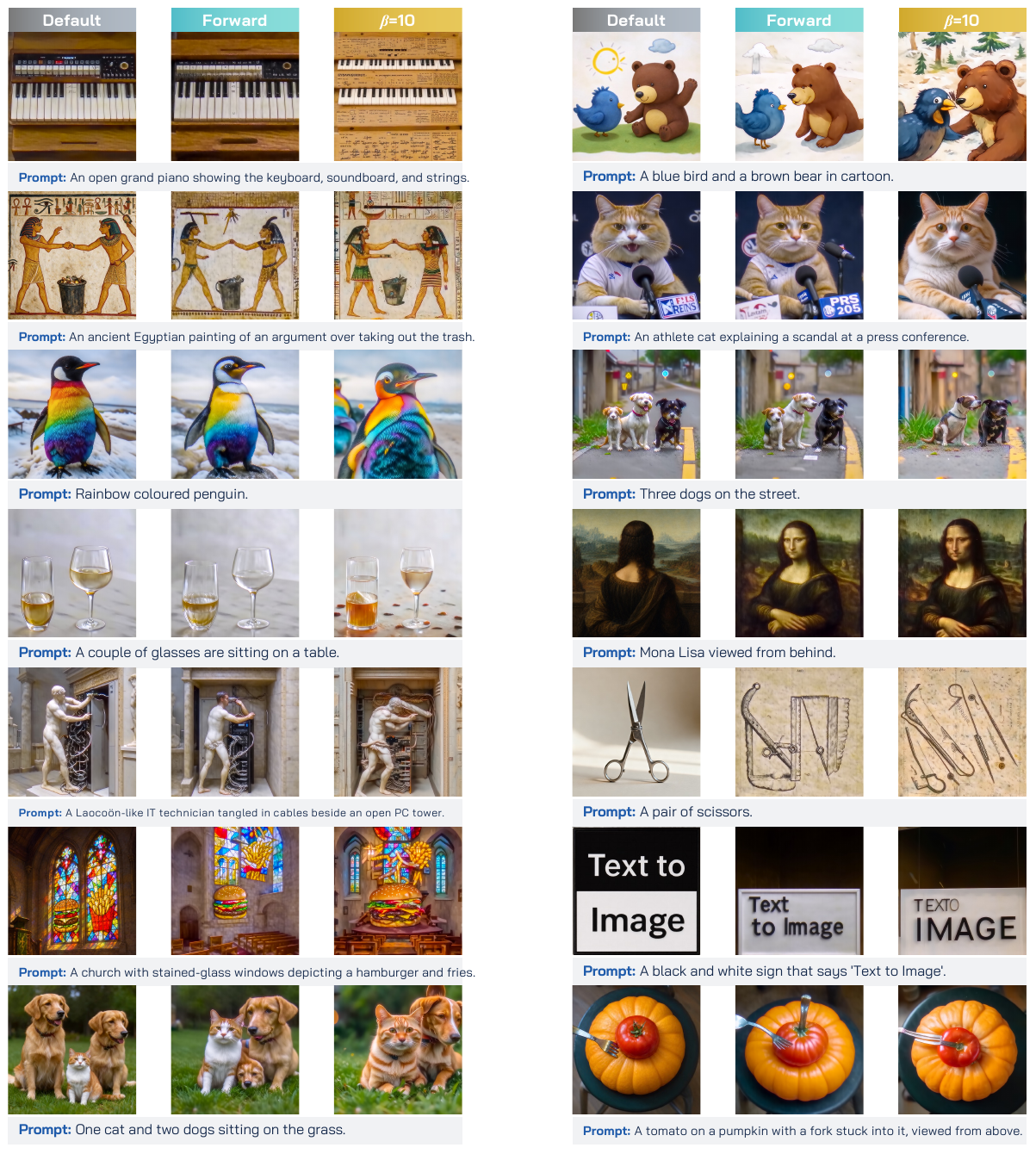}
    \caption{\textbf{Qualitative query-source and branch-scaling ablation.} On-policy rollout queries preserve finer prompt semantics and structure than forward-noised queries, whereas excessive branch scaling degrades layout coherence.}
    \label{fig:ablation_mechanism_qual}
\end{figure}

\begin{figure}[p]
    \centering
    \includegraphics[width=\linewidth,height=0.90\textheight,keepaspectratio]{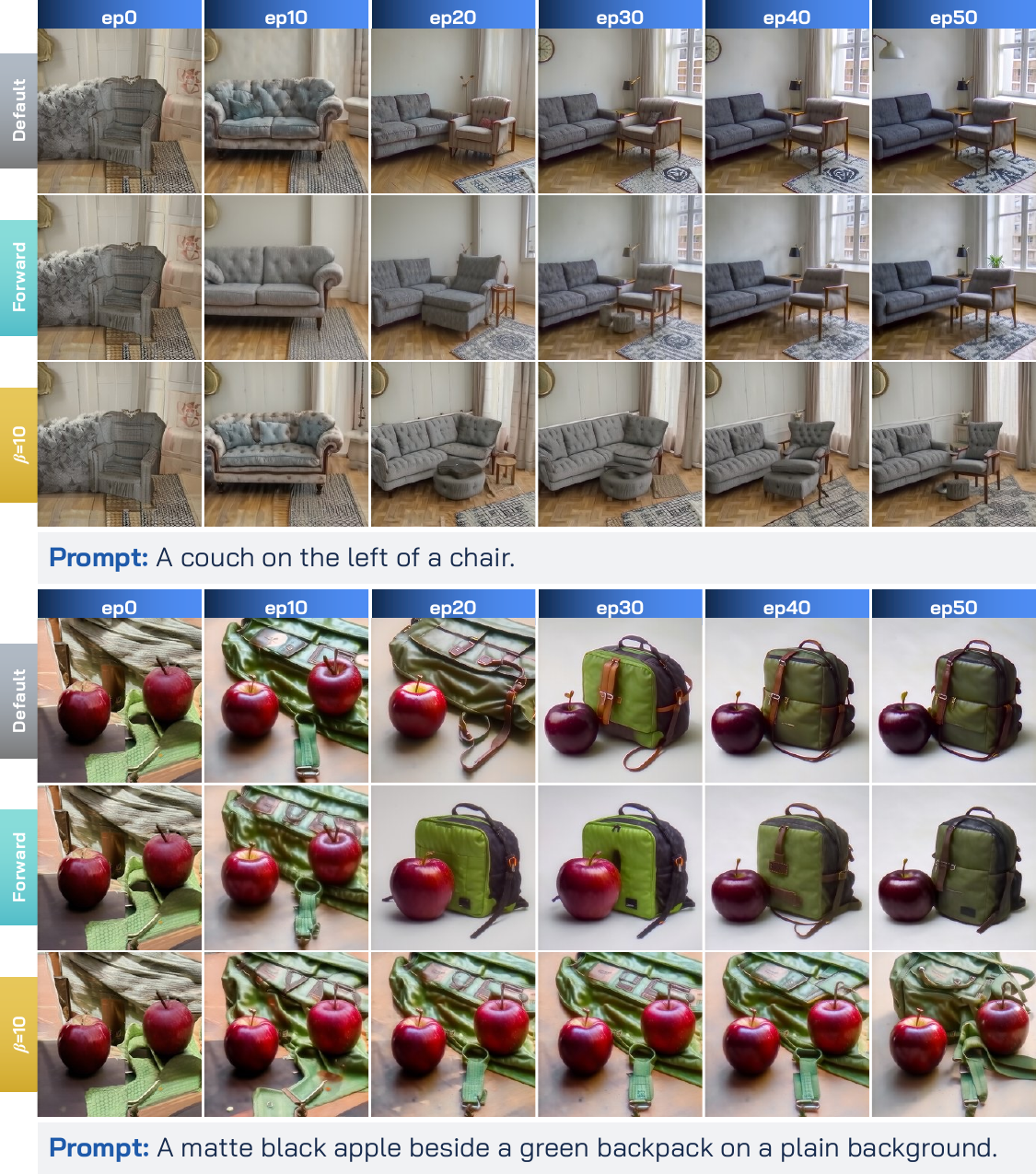}
    \caption{\textbf{Component evolution across training updates.} On-policy rollout queries progressively sharpen subject structure over forward-noised controls, whereas excessive branch scaling induces semantic drift.}
    \label{fig:ablation_progress_mechanism_p02}
\end{figure}

\begin{figure}[p]
    \centering
    \includegraphics[width=\linewidth,height=0.90\textheight,keepaspectratio]{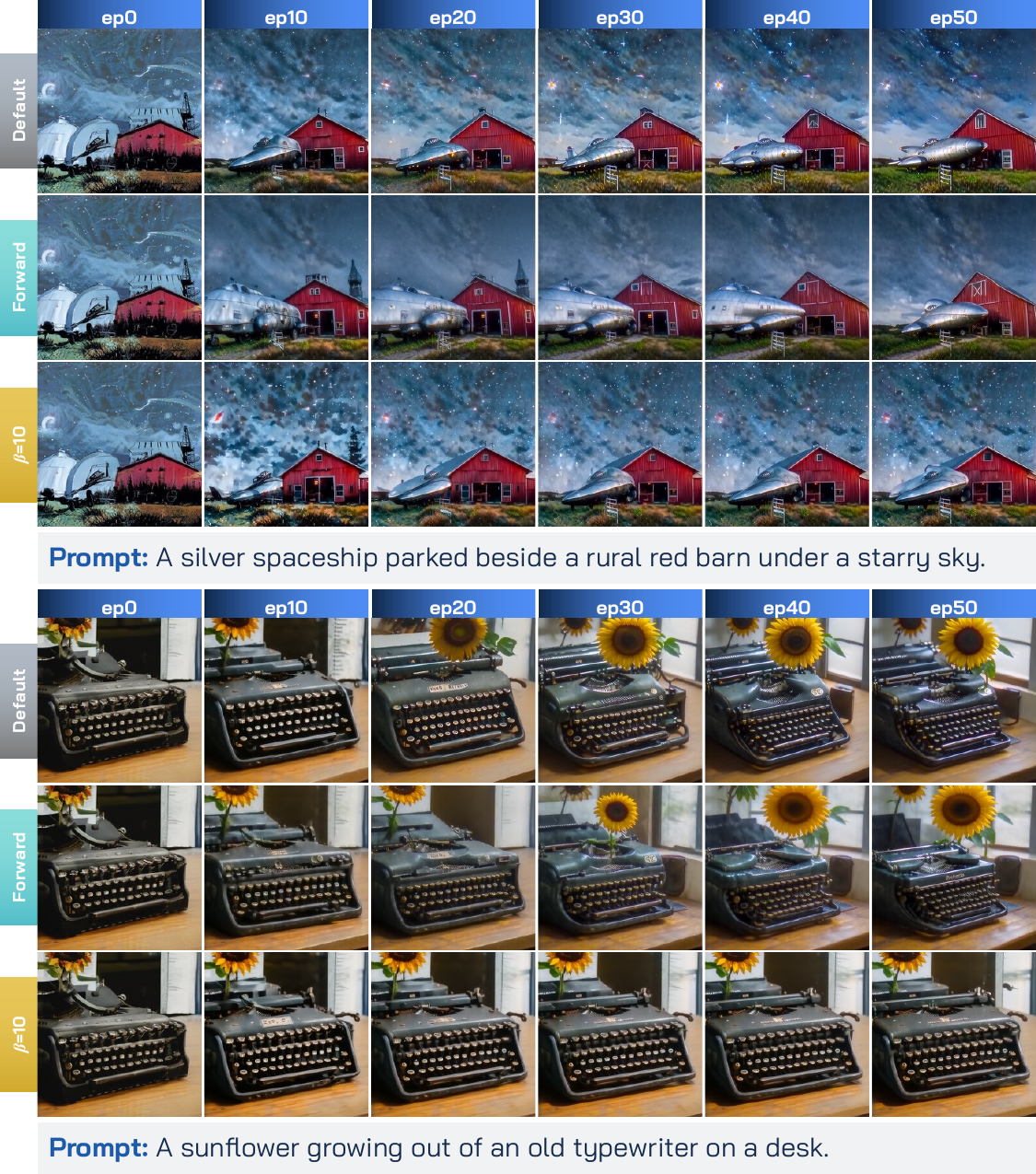}
    \caption{\textbf{Additional component evolution comparisons.} On-policy rollout queries maintain cleaner scene geometry and typography across training updates than forward-noised queries and excessive branch scaling.}
    \label{fig:ablation_progress_mechanism_p03}
\end{figure}

\begin{figure}[p]
    \centering
    \includegraphics[width=\linewidth,height=0.92\textheight,keepaspectratio]{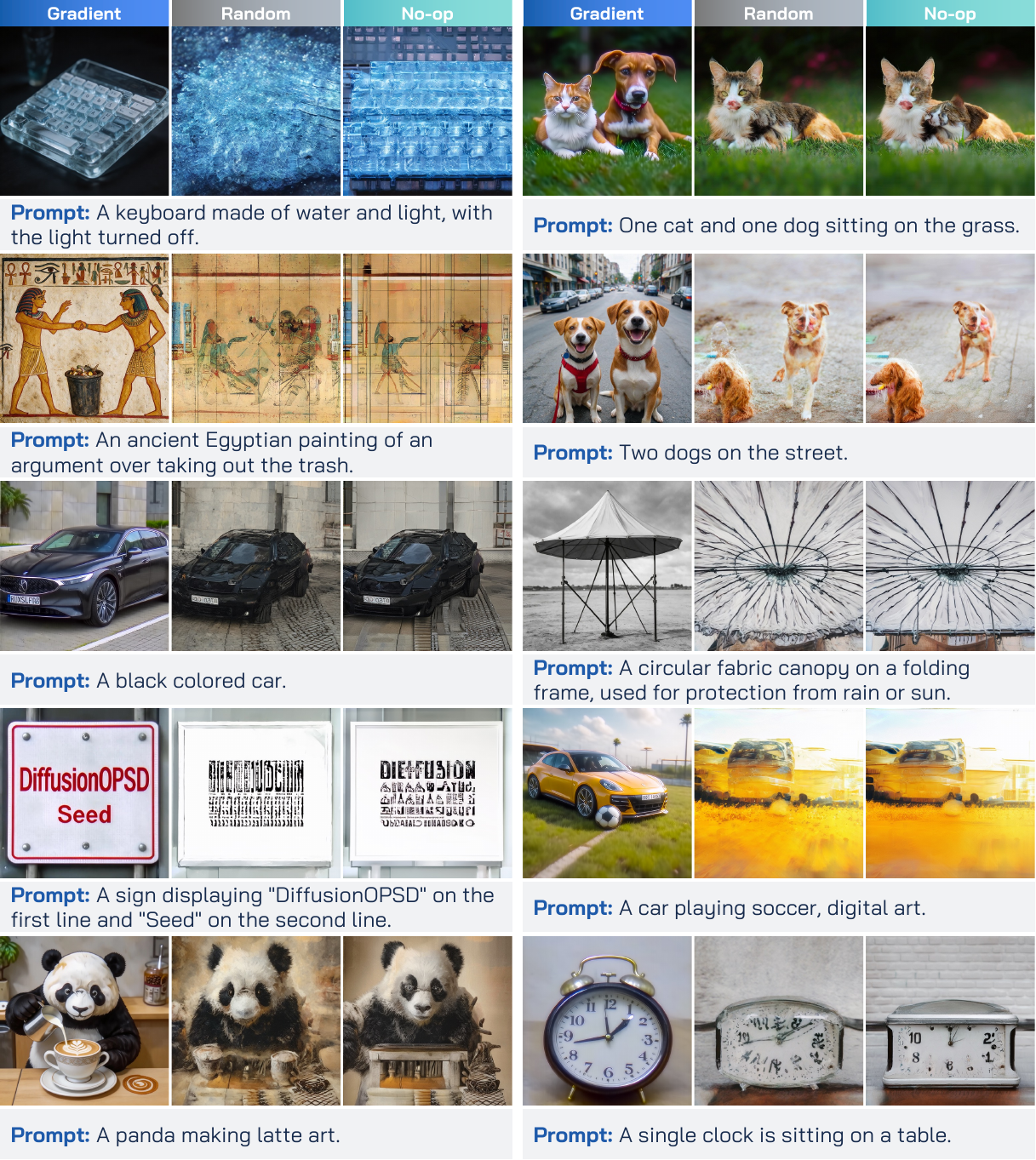}
    \caption{\textbf{Qualitative self-teacher direction ablation.} Reward-gradient self-teachers preserve requested entities, text rendering, and visual style on held-out prompts, whereas random-direction and no-op controls under-edit or drift.}
    \label{fig:ablation_direction_qual}
\end{figure}

\begin{figure}[p]
    \centering
    \includegraphics[width=\linewidth,height=0.94\textheight,keepaspectratio]{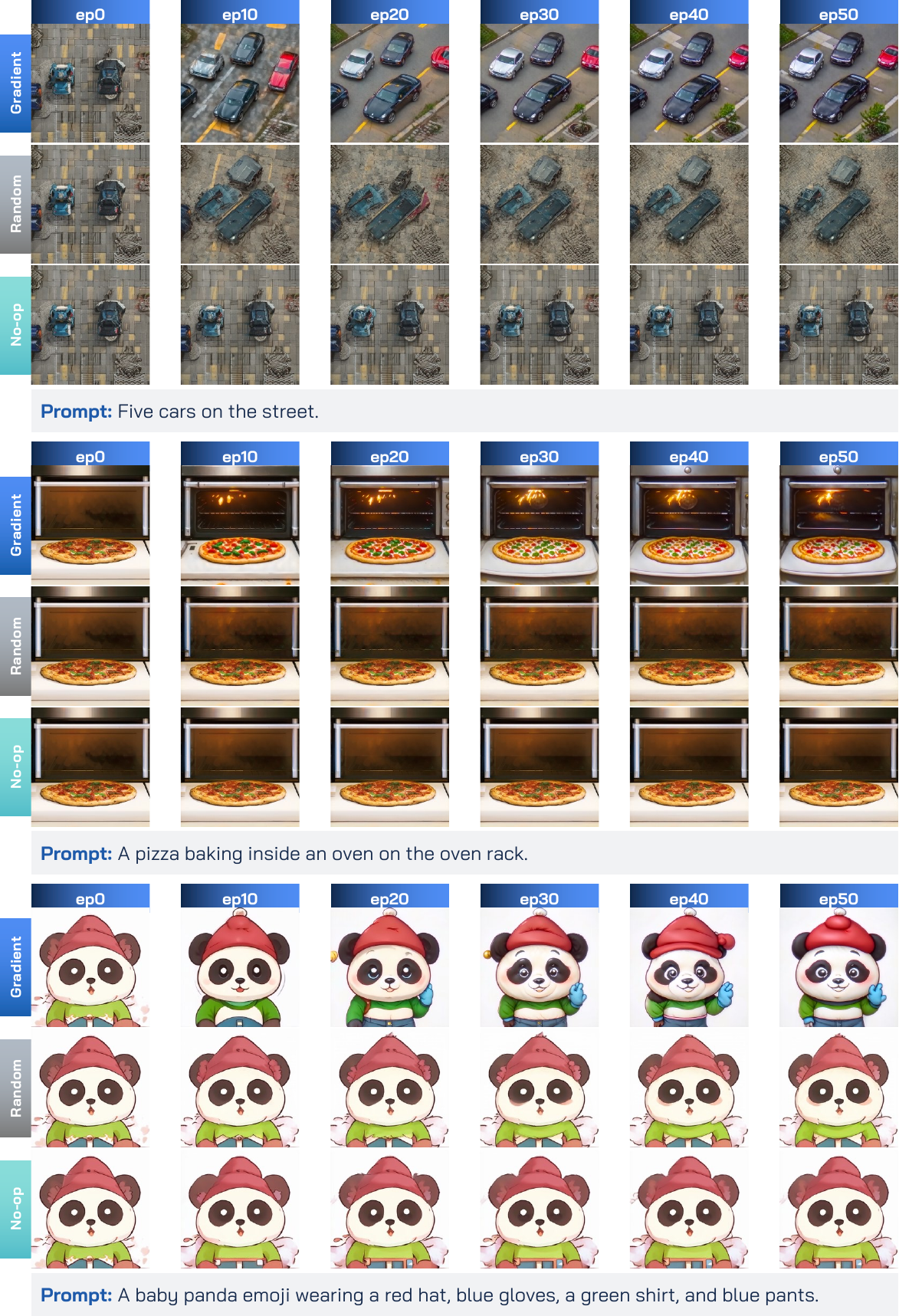}
    \caption{\textbf{Evolution of self-teacher direction variants.} Across training updates on fixed prompts, reward-gradient self-teachers refine target attributes early, whereas random and no-op controls remain near initialization.}
    \label{fig:ablation_progress_direction}
\end{figure}